\documentclass{article} 
\usepackage{iclr2025_conference,times}

\usepackage{amsmath,amsfonts,bm}

\def\eqref#1{equation~\ref{#1}}

\def\1{\bm{1}}

\DeclareMathAlphabet{\mathsfit}{\encodingdefault}{\sfdefault}{m}{sl}
\SetMathAlphabet{\mathsfit}{bold}{\encodingdefault}{\sfdefault}{bx}{n}

\usepackage{url}

\title{Scale-Free Graph-Language Models}

\usepackage{arydshln}

\usepackage{graphicx}
\usepackage{subfigure}
\usepackage{wrapfig}
\usepackage{bm}
\usepackage{multirow}
\usepackage{makecell}
\usepackage{ulem}
\usepackage{amsthm} 
\usepackage{enumitem}
\usepackage{pifont}
\usepackage{algorithm}
\usepackage{algpseudocode}

\newtheorem{proposition}{Proposition}
\newtheorem{definition}{Definition}
\newcommand{\circlednum}[1]{\ding{\numexpr171+#1\relax}} 
\usepackage{mdframed} 
\usepackage{hyperref}
\hypersetup{
    colorlinks=true,
    linkcolor=black,
    filecolor=black,      
    urlcolor=black,
    citecolor=black,
}
\usepackage{booktabs}  
\usepackage{natbib}

\usepackage{xcolor}
\definecolor{darkred}{rgb}{0.6, 0.0, 0.0} 
\definecolor{darkblue}{rgb}{0.0, 0.0, 0.6}

\author{Jianglin~Lu$^{1}$\thanks{Corresponding author: \texttt{jianglinlu@outlook.com}} \ \  \ Yixuan Liu$^{2}$\ \ Yitian Zhang$^{1}$\ \  Yun Fu$^{1,3}$  \\
$^{1}$Department of Electrical and Computer Engineering, Northeastern University \\
$^{2}$Network Science Institute, Northeastern University\\
$^{3}$Khoury College of Computer Science, Northeastern University
}
\iclrfinalcopy 
\begin{document}

\maketitle

\begin{abstract}

Graph-language models (GLMs) have demonstrated great potential in graph-based semi-supervised learning. A typical GLM consists of two key stages: graph generation and text embedding, which are usually implemented by inferring a latent graph and finetuning a language model (LM), respectively. However, the former often relies on artificial assumptions about the underlying edge distribution, while the latter requires extensive data annotations.
To tackle these challenges, this paper introduces a novel GLM that integrates graph generation and text embedding within a unified framework.
Specifically, for graph generation, we leverage an inherent characteristic of real edge distribution—the \textit{scale-free property}—as a structural prior. 
We unexpectedly find that this natural property can be effectively approximated by a simple $k$-nearest neighbor (KNN) graph.
For text embedding, we develop a graph-based pseudo-labeler that utilizes scale-free graphs to provide complementary supervision for improved LM finetuning.
Extensive experiments on representative datasets validate our findings on the scale-free structural approximation of KNN graphs and demonstrate the effectiveness of integrating graph generation and text embedding with a real structural prior.
Our code is available at \href{https://github.com/Jianglin954/SFGL}{\textcolor{black}{https://github.com/Jianglin954/SFGL}}.
\end{abstract}

\section{Introduction}
Recently, graph-language models (GLMs) have been widely explored in graph-based semi-supervised classification on documents, especially for citation networks \citep{qin2023disentangled, yu2023empower, jianglin2023LGI, he2023harnessing}. 
When designing a GLM for classification, two key challenges arise: {\textit{graph generation}}—how to generate a reasonable graph structure for the given documents, and {\textit{text embedding}}—how to encode the textual sequences into meaningful semantic features.
To address these problems, various GLMs have been proposed, which can be broadly categorized into latent graph inference (LGI) models and language-assisted graph (LAG) models.

LGI models focus on graph generation and typically rely on feature engineering approaches, such as bag-of-words~\citep{bow}, TF-IDF \citep{AIZAWA200345}, and skip-gram~\citep{skipgram}, to encode textual sequences into shallow representations. 
These representations serve as the foundation for optimization objectives that jointly learn the underlying graph structure and node embeddings using graph neural networks (GNNs) \citep{LDS, SLAPS, jianglin2023LGI, DGM}.
\textit{However, LGI models typically rely on artificial assumptions about the underlying edge distribution, which may not accurately reflect the real graph structure}.

On the other hand, LAG models assume a predefined graph structure and focus on enhancing text embedding through powerful language models (LMs), such as DeBERTa \citep{he2021deberta} and GPT \citep{chatgpt}.
The core idea behind LAG models is to finetune LMs on the target datasets, thereby generating more informative text embeddings \citep{yu2023empower, duan2023simteg, he2023harnessing}.  
\textit{However, finetuning a pretrained LM typically requires target datasets with extensive annotations \citep{GPT3}, which is often impractical for semi-supervised learning tasks. When sufficient annotations are unavailable, LM finetuning may become unreliable.}

In this paper, we propose a novel {scale-free graph-language (SFGL)} model that integrates graph generation and text embedding into a unified framework. {The proposed SFGL framework addresses the challenges of relying on specious edge distribution assumptions and mitigates the limitations of insufficient supervision in LM finetuning}. This is accomplished by identifying an inherent structural nature of real networks for rational graph generation and incorporating a graph-based pseudo-labeler to enhance LM finetuning.
Specifically, we investigate a fundamental characteristic of edge distribution in real networks: the \textit{scale-free property} \citep{barabasi2003scale, radicchi2011citation, barabasi2016network}, which is primarily characterized by a few highly connected hubs and numerous small nodes with few edges (see Fig. \ref{powerlaw} for an example). 
The scale-free edge distribution is prevalent in real-world networks, particularly in citation networks, making it an ideal structural prior for graph generation. However, the dynamic and sequential nature of its formation process makes efficiently fitting this distribution highly challenging.
Fortunately, we reveal that \textit{a $k$-nearest neighbor (KNN) graph, constructed using cosine similarity as the distance metric and an appropriately chosen $k$, closely approximates a scale-free network}.
We subsequently leverage this scale-free graph to develop a graph-based pseudo-labeler, which provides complementary supervision for enhancing LM finetuning. 
Consequently, the improved LMs can generate semantically enriched text embeddings. 
Note that our proposed SFGL model can be trained iteratively to further enhance performance. 
We conduct extensive experiments to validate our findings and highlight the potential of GLMs built on scale-free structures.
Our key contributions are summarized as follows:

\begin{itemize}
\item We identify two key challenges in existing GLMs: {artificial structural assumptions in graph generation} and {unreliable LM finetuning for text embedding}. We propose addressing these challenges simultaneously by exploring a well-grounded graph structural prior.

\item We leverage the scale-free edge distribution in real networks as our graph structural prior. Our empirical validation and analysis reveal that a KNN graph, constructed using cosine similarity with an appropriately chosen $k$, effectively approximates a scale-free network.

\item To the best of our knowledge, the proposed SFGL is the first work to unify graph generation and text embedding within a GLM framework, highlighting the synergistic potential of GNNs and LMs under a scale-free structural prior.
\end{itemize}

\section{Preliminaries}
\subsection{Problem Definition}
\label{problem_define}
Given a set of documents denoted as $\mathcal{D}=\{\mathcal{X}_L, \mathcal{X}_U, \mathcal{Y}_L\}$, 
where $\mathcal{X}_L=\{\mathbf{x}_i\}^{m}$ and $\mathcal{X}_U=\{\mathbf{x}_i\}^{n}$ represent the text sequences of $m$ labeled and $n$ unlabeled documents, respectively, and $\mathcal{Y}_L=\{\mathbf{y}_i\}^{m}$ is the label set for the labeled samples, with $m\ll n$ indicating a scarcity of labeled nodes compared to the abundance of unlabeled ones, we tackle a graph-based semi-supervised classification task that involves training a model on $\mathcal{D}$ to classify unlabeled documents $\mathcal{X}_U$. 
Here, we face two key problems: {\textit{graph generation}}—how to generate a reasonable graph structure for the given documents, and {\textit{text embedding}}—how to encode the textual sequences into meaningful semantic features.

For graph generation, the main challenge lies in identifying a reasonable structural prior that characterizes the underlying distribution of edges. The choice of this prior is critical, as it determines both the complexity of the learning model and the inductive bias introduced. We suggest that an effective 
\begin{wrapfigure}{r}{5.4cm}
\centering
\includegraphics[scale=0.295,trim=36 25  40 40,clip]{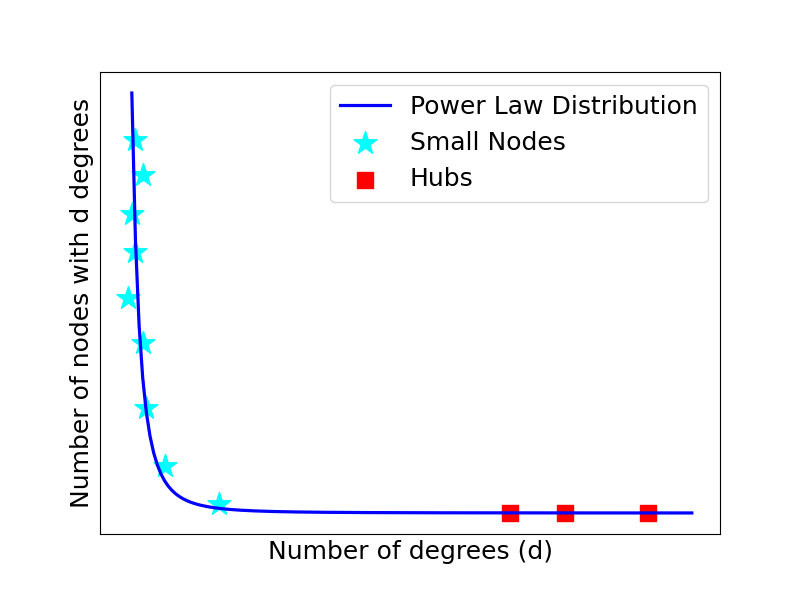}
\caption{Illustration of a scale-free network, highlighting a few \textit{hubs} and a large number of \textit{small nodes}.}
\vspace{-2mm}
\label{powerlaw}
\end{wrapfigure}
structural prior should exhibit two key characteristics: \textit{real}—it should reflect an inherent attribute of the data; and \textit{simple}—it should allow for straightforward graph generation without requiring additional model training.

For text embedding, advanced LMs are preferred over traditional feature engineering approaches for their superior ability to capture semantic information from textual sequences of documents. 
The primary challenge here lies in improving LM finetuning in a semi-supervised setting with very limited labeled samples.
We suggest addressing this challenge with a graph-based semi-supervised learning model that propagates supervised information from labeled to unlabeled samples through the graph structure. 
However, this brings us back to the central question: how to identify a reasonable structural prior for graph generation.

\subsection{Scale-Free Networks}
\label{SFN}
We investigate a fundamental structural nature of real-world networks: the \textit{scale-free property} \citep{barabasi1999emergence, barabasi2003scale}.
In essence, real citation networks exhibit scale-free characteristics \citep{Redner_1998, hummon1989connectivity, radicchi2011citation}. Their degree distribution $\mathbb{P}(\theta)$ follows a power law:
$\mathbb{P}(\theta) \propto \theta^{-\alpha}$,
where $\theta$ denotes the degree of a node and $\alpha$ is a scaling parameter of the distribution \citep{clauset2009power}.
For illustration, Fig. \ref{powerlaw} depicts the degree distribution of a scale-free network with $\alpha=2.8$.
As observed, the distribution is heavy-tailed, with a few highly connected nodes, known as \textit{hubs}, and many \textit{small nodes} with only a few edges \citep{barabasi1999emergence, barabasi2016network}.
We suggest that leveraging the scale-free property for graph generation is a strong choice, as it reflects an inherent characteristic of citation networks and aligns with our \textit{real} principle.
{The challenge lies in constructing a latent graph that can (approximately) form a scale-free network.}

\begin{figure*}[!tb]
    \centering
    \vspace{-8mm}
    \subfigure[$k=2$]{\includegraphics[scale=0.30,trim=30 400  620 30,clip]{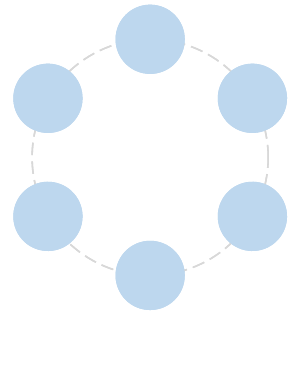}\label{K2}} 
    \quad \  \subfigure[$k=3$]{\includegraphics[scale=0.30,trim=35 380  530 30,clip]{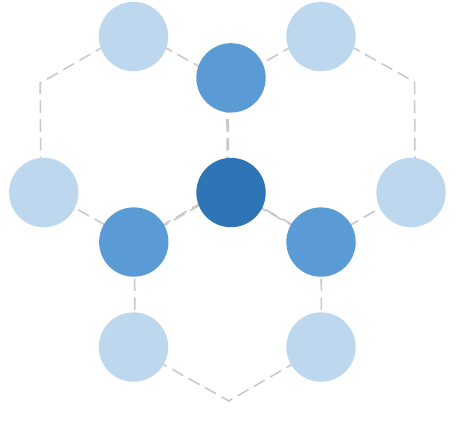}\label{K3}} 
    \quad \ \subfigure[$k=4$]{\includegraphics[scale=0.30,trim=30 360  530 30,clip]{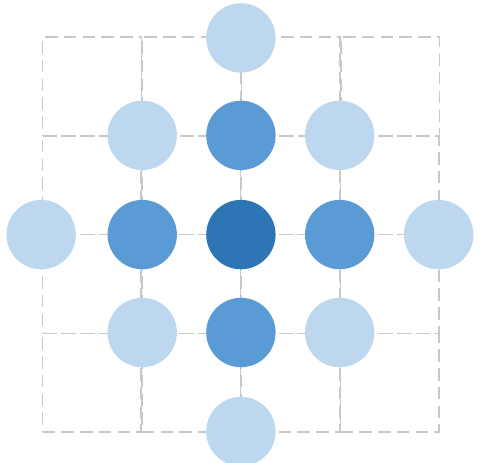}\label{K4}}
    \quad \ \subfigure[$k=6$]{\includegraphics[scale=0.30,trim=25 300  450 30,clip]{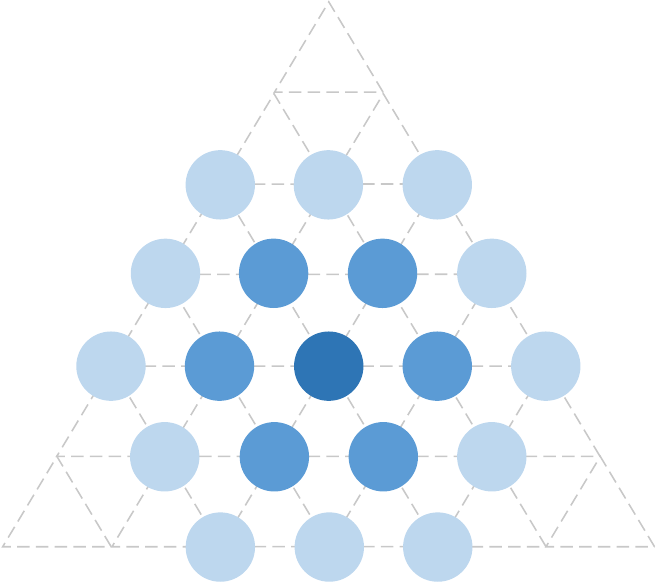}\label{K6}}
    \vspace{-2mm}
    \caption{Uniform node distribution within structured arrangements causes consistent connectivity.}
    \label{knngraph}
     \vspace{-3mm}
\end{figure*}

\subsection{K-Nearest Neighbor Graphs}
\label{KNN}
We begin by exploring KNN graphs, as their construction aligns with our \textit{simple} principle.
The construction process is straightforward: for each node, we identify its $k$-nearest neighbors using a specified distance metric based on node features and establish an edge between the node and each of its $k$ nearest neighbors.
The degree distribution of KNN graphs is primarily shaped by three key factors: \circlednum{1} \textit{node distribution}, which is unknown a priori; \circlednum{2} \textit{the choice of similarity metric}, which affects the probability of node connections; and \circlednum{3} \textit{the value of $k$}, which directly determines node degrees. 
Without detailed knowledge on these factors, characterizing the degree distribution of KNN graphs remains challenging. 
We first explore simple cases and analyze each factor independently.

In the absence of prior knowledge about the underlying node distribution, we analyze attribute-free nodes that are uniformly distributed across regular grid structures, including circles, hexagons, quadrilaterals, and triangles.
As shown in Fig. \ref{knngraph}, the subgraph formed by any node and its first-order neighbors exhibits characteristics of a random network \citep{barabasi1999emergence}. Within a specified value of $k$, the connection probabilities between two nodes within the subgraph remain identical. Thus, building KNN graphs on such uniform distributions yields consistent node connectivity.

In KNN graphs, various distance metrics can be used to measure similarity between nodes, including Euclidean, Manhattan, and Cosine distances.
Euclidean distance is sensitive to variations in feature scale and highly influenced by outliers. Manhattan distance is more robust to outliers but disregards the actual geometric relationships between data points. Cosine distance is particularly suitable for high-dimensional, sparse data, as it prioritizes vector direction over magnitude.

The value of $k$ directly determines the magnitude of node degrees. As $k$ increases, more edges are formed, leading to higher node degrees. In the extreme case where $k=m+n-1$, the KNN graph becomes a complete graph, where every node is connected to all others and has the same degree.

\section{Scale-Free Graph-Language Model}

\subsection{Scale-Free Graph Generation}
\label{SFG}

\textbf{Hypothesis \& Analysis}.
In this paper, we argue that constructing a KNN graph that approximates a scale-free network is practically feasible. To support this, we first present the following hypothesis:

\begin{mdframed}[linecolor=black,linewidth=0.5pt]
\textit{A $k$-nearest neighbor graph, when constructed on citation networks using cosine similarity as the distance metric and an appropriately chosen $k$, approximates a scale-free network.}
\end{mdframed}

To investigate the proposed hypothesis, we present the following definitions and propositions.
\begin{definition}[]
\label{Degree}
Given an undirected graph $\mathcal{G}:=(\mathcal{V}, \mathcal{E})$ with $|\mathcal{V}|$ nodes and $|\mathcal{E}|$ edges, the degree centrality of a node $v$ is $C_D(v) = \theta(v)$, where $\theta(v)$ is the degree of $v$.
\end{definition}

\begin{definition}[]
\label{def:inj}
Given a directed graph $\mathcal{G}:=(\mathcal{V}, \mathcal{E})$ with $|\mathcal{V}|$ nodes and $|\mathcal{E}|$ edges, the in-degree centrality and out-degree centrality of a node $v$ are $C_{ID}(v) = \theta_i(v)$ and $C_{OD}(v) = \theta_o(v)$, where $\theta_i(v)$ and $\theta_o(v)$ are the in-degree and out-degree of $v$, respectively.
\end{definition}

\begin{proposition}
For a scale-free network, assuming that there is at most one node $v$ whose degree belongs to $[\theta_{max}, \infty)$, we have $\theta_{max} = \theta_{min} |\mathcal{V}|^{\frac{1}{\alpha-1}}$, where $\theta_{max}$ and $\theta_{min}$ refer to the maximum and minimal degree in the scale-free network.
\label{prop2}
\end{proposition}

\begin{proposition}
For a KNN graph, we have $\sum_{v\in \mathcal{V}}C_{ID}(v) = \sum_{v\in \mathcal{V}}C_D(v)  - \sum_{v\in \mathcal{V}}C_{OD}(v) = 2 |\mathcal{E}|  - k |\mathcal{V}| $, where $|\mathcal{E}| \leq k |\mathcal{V}|$.
\end{proposition}

Considering their inherent biases, KNN graphs exhibit significant similarities to scale-free networks, which we detail as follows.
\circlednum{1} \textit{Homophily:} citation networks exhibit a high degree of homophily, 
\begin{wrapfigure}{r}{5.9cm}
\centering
\includegraphics[scale=0.30,trim=0 0  0 50,clip]{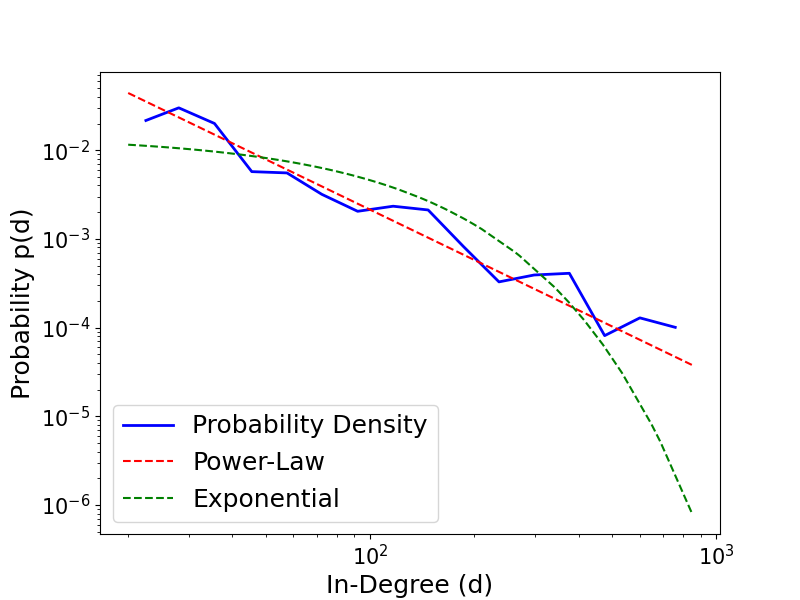}
\vspace{-4mm}
\caption{Curve fitting to the in-degree distribution of a KNN graph ($k=5$) using Euclidean distance as the metric.} 
\vspace{-2.5mm}
\label{Euclidean}
\end{wrapfigure}
where connected nodes usually have a high probability of belonging to the same class. Notably, the KNN algorithm is also based on the homophily assumption that similar nodes exist in close proximity.
\circlednum{2} \textit{Directed:} citation networks are inherently directed since citations have a clear direction from the citing publication to the cited publication. Coincidentally, the KNN graph is also directed as node $v$ being in the $k$-nearest neighbors of node $u$ does not imply node $u$ is in the $k$-nearest neighbors of node $v$.
\circlednum{3} \textit{Upper bound of degree}: the size of hubs in both scale-free networks and KNN graphs is constrained.
Proposition \ref{prop2} indicates that the size of hubs in scale-free networks has an upper bound of $d_{min} |\mathcal{V}|^{\frac{1}{\alpha-1}}$. Note that the maximal in-degree in a KNN graph is also bounded, and the out-degree of each node is always fixed as $k$, as shown in proposition \ref{prop1}.

We discuss the underlying reasons and rationale for using ``cosine similarity'' and the term  ``approximates''.
Cosine similarity focuses on the orientation of feature vectors rather than their magnitude, aligning with the syntactic patterns of word semantics. 
Consequently, few hubs with high in-degrees will appear in KNN graphs if the feature space exhibits some common orientations.
The term ``approximates'' is used because KNN graphs are constructed based on node features, which are influenced by the feature engineering approaches employed. After feature extraction, the node distribution may not accurately represent the original data distribution of citation networks.

\textbf{Empirical Validation}.
To validate our hypothesis, we construct KNN graphs using cosine similarity with $k=5$ and examine their degree distribution across multiple citation networks, including \texttt{Cora}, \texttt{Pubmed}, and \texttt{ogbn-arxiv} (see Appendix \ref{datasetdescription} for details). 
Figure \ref{in_degree_distribution} presents the in-degree distribution of KNN graphs on these datasets.
Our observations reveal that the generated KNN graphs contain massive small nodes alongside a few hubs,  resulting in in-degree distributions that resemble those of scale-free networks as depicted in Fig. \ref{powerlaw}. 
Notably, the degree distributions exhibit slight variations across datasets, which can be attributed to differences in node distributions and the feature extraction methods used to construct feature vectors.
To further investigate these degree distributions, we fit the in-degree distributions to exponential and power-law models—both characterized by heavy-tailed behavior. 
Figure \ref{fitting_curves} illustrates the fitting curves on log-log axes using logarithmically spaced bins. 
As shown, the in-degree distributions of KNN graphs align more closely with a power law, with estimated scaling exponents ($\alpha$) of $3.3$, $2.8$, and $2.7$ for \texttt{Cora}, \texttt{Pubmed}, and \texttt{ogbn-arxiv}, respectively. These empirical findings suggest that KNN graphs effectively capture the structural properties of citation networks, thus supporting our hypothesis.

We further investigate the degree distribution of KNN graphs using Euclidean distance and present the fitting curves in Fig. \ref{Euclidean}.
As shown, the resulting degree distribution does not fit well to a power-law model, indicating that Euclidean distances between vector points fail to accurately capture the similarities between textual features in citation networks.
This is because Euclidean distance primarily emphasizes the magnitude of vectors in linear space, neglecting important elements such as the angular distribution or relative orientation of vectors, which are essential for capturing syntactic patterns and semantic relationships of textual attributes. More comparisons are provided in Sec. \ref{ablation}.

\begin{figure*}[!tb]
    \subfigure[\texttt{Cora}]{\includegraphics[scale=0.245,trim=0 0  50 45,clip]{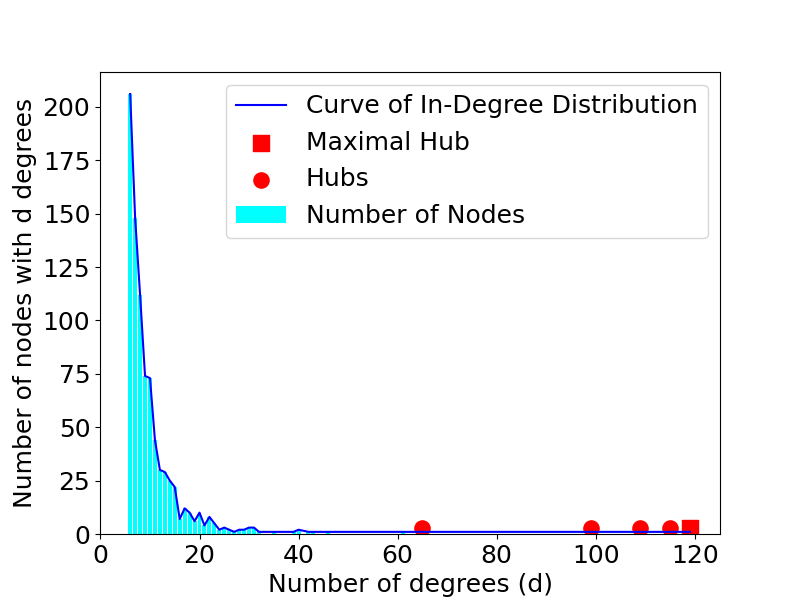}\label{fig-cora20}} 
    \subfigure[\texttt{Pubmed}]{\includegraphics[scale=0.245,trim=0 0  50 45,clip]{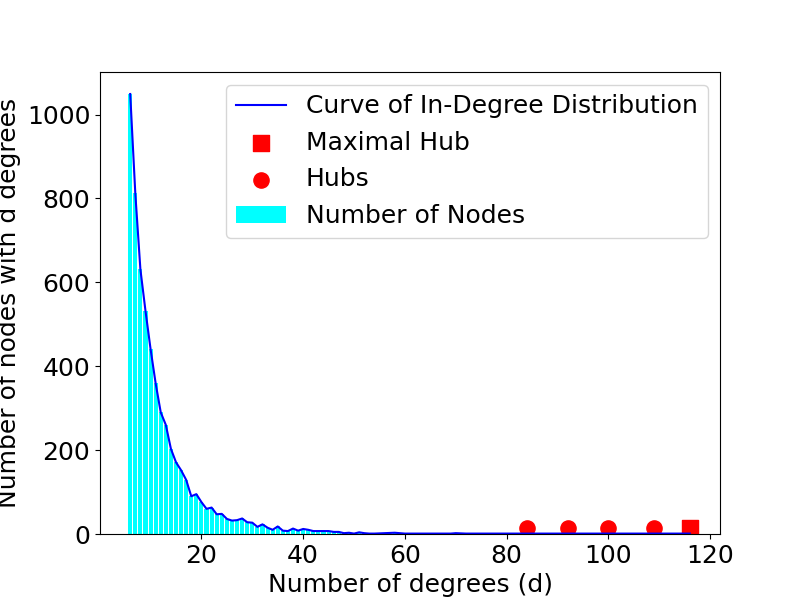}\label{fig-pubmed20}} 
    \subfigure[\texttt{ogbn-arxiv}]{\includegraphics[scale=0.245,trim=0 0  50 45,clip]{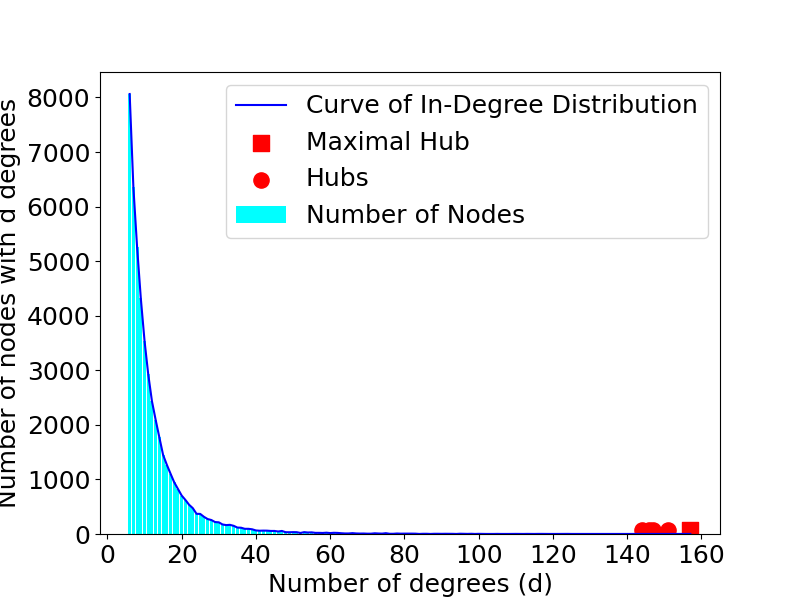}\label{fig-ogbn-arxiv20}}
    \vspace{-2mm}
    \caption{In-degree distribution of KNN graphs on different datasets, where the top-$5$ largest hubs are marked in red.}
    \label{in_degree_distribution}
     \vspace{-1mm}
\end{figure*}
 \begin{figure*}[!tb]
     \subfigure[\texttt{Cora}] {\includegraphics[scale=0.245,trim=0 0  50 50,clip]{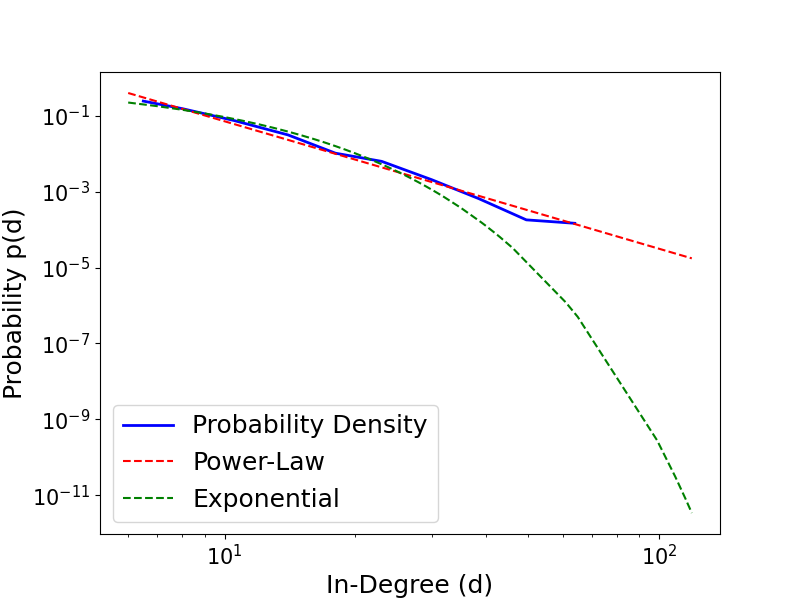}}
    \subfigure[\texttt{Pubmed}]{\includegraphics[scale=0.245,trim=0 0  50 50,clip]{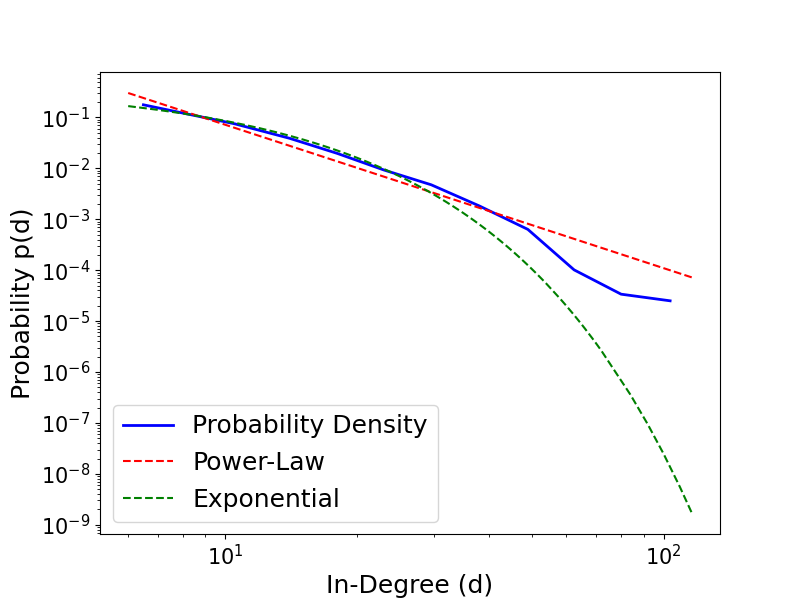}}
     \subfigure[\texttt{ogbn-arxiv}]{\includegraphics[scale=0.245,trim=0 0  50 50,clip]{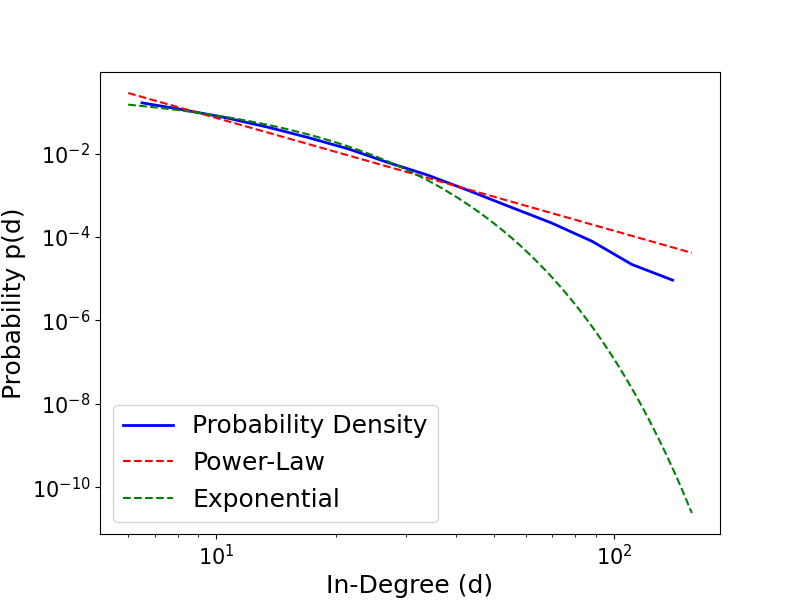}}
      \vspace{-2mm}
     \caption{Curve fitting to the in-degree distribution of KNN graphs on different datasets. }
     \vspace{-1mm}
     \label{fitting_curves}
 \end{figure*}

\textbf{Graph Construction}.
Building on the above analysis and validation, we generate a scale-free KNN graph for the given documents $\mathcal{D}$. We begin by encoding the text attributes $\mathcal{X}=\{\mathcal{X}_L \cup \mathcal{X}_U\}=\{\mathbf{x}_i\}^{m+n}$ into hand-crafted feature representations using a text encoder $\texttt{Enc}(\cdot)$: $\mathbf{f}_i=\texttt{Enc}(\mathbf{x}_i)$, where $\mathbf{f}_i$ represents the extracted feature vector for $\mathbf{x}_i$, and $\texttt{Enc}(\cdot)$ can be either a traditional model or a language model. 
Next, we construct a KNN graph on features $\mathbf{F}$ using cosine distance: $\mathbf{A}=\texttt{KNN}_{\texttt{cos}}(\mathbf{F}, k)$, where $\mathbf{F}=[\mathbf{f}_1, \mathbf{f}_2, \ldots, \mathbf{f}_{m+n}]$, $\mathbf{A}$ is the adjacency matrix of the resulting graph, $k$ is the number of nearest neighbors, and $\texttt{KNN}_{\texttt{cos}}(\cdot)$ denotes the KNN algorithm with cosine distance.
For simplicity, we consider a binary graph, \textit{i.e.,} $\mathbf{A}_{ij}=1$ if $\mathbf{A}_{ij} \in \texttt{topk}(\mathbf{A}_{i:})$; $\mathbf{A}_{ij}=0$, otherwise.

\subsection{Improved Text Embedding}
\label{TFE}
The previous section addresses the challenge in graph generation, as outlined in Sec. \ref{problem_define}. This section focuses on the challenge in text embedding, i.e., improving the finetuning of LMs. 

\textbf{Pseudo-Labeler}.
Leveraging the graphs generated in Sec. \ref{SFG}, we design a graph-based semi-supervised learning model to produce high-quality pseudo-labels, which enhance LM training.
Given the node features $\mathbf{F}$, adjacency graph $\mathbf{A}$, and available labels $\mathcal{Y}_L=\{\mathbf{y}_i\}^{m}$, we train a GNN using the standard node classification loss:
\begin{equation}
    \mathcal{L}_{\texttt{GNN}} = \sum_{i=1}^{m}\mathbf{y}_i^\top \ln \mathbf{z}_i, \quad  \mathbf{z}_i = \operatorname{softmax}(\texttt{GNN}_{\mathbf{\Phi}}(\mathbf{f}_i, \mathbf{A})),
\label{gnn_train}
\end{equation}
where $\mathbf{\Phi}$ represents the GNN parameters. 
Here, any GNN model, such as graph convolutional networks (GCN) \citep{kipf2016semi}, GraphSAGE \citep{hamilton2017inductive}, or graph attention networks (GAT) \citep{GAT}, can be employed (see Sec. \ref{ablation} for experimental comparisons).
Once trained, the pseudo labels $\hat{\mathbf{y}}_j$ for all unlabeled nodes ($j \in \{1, \ldots, n\}$) are generated by $\hat{\mathbf{y}}_j = \operatorname{argmax}(\texttt{GNN}_{\mathbf{\Phi}}(\mathbf{f}_j, \mathbf{A}))$, enabling effective label propagation.

\textbf{Text Embedding}.
\label{LTG}
To extract semantic-enriched text embeddings from documents, we finetune LMs on $\hat{\mathcal{D}}=\{\mathcal{X}_L, \mathcal{X}_U, \mathcal{Y}_L, \hat{\mathcal{Y}}_U\}$ 
where $\hat{\mathcal{Y}}_U = \{\hat{\mathbf{y}}_j\}^{n}$, using the following loss function: 
\begin{equation}
    {\mathcal{L}}_{\texttt{LM}} = \frac{1}{m}\sum_{i=1}^{m}\mathbf{y}_i^\top \ln \mathbf{s}_i + \frac{1}{n} \sum_{j=1}^{n}\hat{\mathbf{y}}_j^\top \ln \mathbf{s}_j,
\label{lm_finetune}
\end{equation}
where $\mathbf{s}_i = \operatorname{softmax}(\texttt{LM}_{\mathbf{\Theta}}(\mathbf{x}_i))$ and $\texttt{LM}_{\mathbf{\Theta}}(\cdot)$ is a pretrained LM with parameters $\mathbf{\Theta}$.
After finetuning, for each node $\mathbf{x}_i$, $i\in\{1, 2, \ldots, m+n\}$, we generate its text embedding $\mathbf{e}_i$ by extracting the last hidden states of the finetuned LM, \textit{i.e.,} $\mathbf{e}_i = \texttt{LM}_{\mathbf{\Theta}}(\mathbf{x}_i)$.

Following \citep{he2023harnessing}, we also utilize the large language model (LLM) GPT3.5~\citep{chatgpt} to enhance text embeddings.
We access GPT3.5 through its open API and treat it as an oracle, feeding it a prompt input $\mathbf{p}_i$ and obtaining the corresponding textual response $\mathbf{r}_i$ for each node $i$:
\begin{equation}
    \mathbf{r}_i = \texttt{LLM}(\mathbf{p}_i), \quad \mathbf{p}_i=\operatorname{concat}(\mathbf{x}_i, \mathbf{q}),
\end{equation}
where the prompt $\mathbf{p}_i$ is concatenated by the text sequences $\mathbf{x}_i$ and a task-specific question $\mathbf{q}$ \citep{he2023harnessing}.
The question $\mathbf{q}$ is standardized and universal for all nodes, which prompts GPT3.5 to classify each node and provide an explanation for its prediction (see Appendix for more details). 
Note that, we do not need to finetune GPT3.5 or access its inner structures or embeddings. 
Once the textual response $\mathbf{r}_i$ is obtained for $\mathbf{x}_i$, we replace $\mathbf{x}_i$ with $\mathbf{r}_i$ and finetune the $\texttt{LM}_{\mathbf{\Theta}}(\cdot)$ to generate updated text embeddings for the documents.

\subsection{Classification}
\label{final_class}
Up to this point, we have addressed the two key problems outlined in Sec. \ref{problem_define}, obtaining semantic-enriched text embeddings $\mathbf{E}=[\mathbf{e}_1, \mathbf{e}_2, \ldots, \mathbf{e}_{m+n}]$ and a meaningful graph $\mathbf{A}$.
We can now tackle the classification task by training a graph-based semi-supervised learning model on $\mathbf{E}$ and $\mathbf{A}$ using the available labels $\mathcal{Y}_L$. Following \citep{he2023harnessing}, we utilize a GNN as the classification model. 
We also evaluate the performance of the LMs with generated pseudo-labels. 
Notably, the proposed SFGL can use an iterative training mechanism. The improved text embeddings facilitate the generation of higher-quality pseudo-labels, which, in turn, further enhance the discriminability of the embeddings. 
We explore this iterative training strategy in Sec. \ref{ablation}. Due to space limitations, the detailed training stages are summarized in Appendix \ref{trainingstrategies}.

\section{Experiments}
\label{Experiments}

\begin{table*}[!tb]
    \footnotesize
    \caption{Performance of different models trained on four datasets with varying labeling rates, where ``$---$" indicates unreliable results due to an extremely limited number of training labels.}
    \centering
    \setlength{\tabcolsep}{1.5mm}{ 
    \begin{tabular}{l|ccccc}
        \toprule
         Datasets \& Method & \multicolumn{5}{c}{Labeling Rates (labels)} 
         \\
         \midrule
         \texttt{Cora}
         & $2.58\%_{(70)}$ & $5.17\%_{(140)}$ & $7.75\%_{(210)}$ & $15.51\%_{(420)}$ & $22.16\%_{(600)}$
        \\ 
        \midrule
        \texttt{GCN} 
        & $0.5898_{\pm0.0159}$ 
        & $0.6038_{\pm0.0116}$ 
        & $0.6358_{\pm0.0164}$ 
        & $0.6936_{\pm0.0055}$ 
        & $0.7114_{\pm0.0049}$ 
        \\ 
        \texttt{DeBERTa}  
        & $0.2406_{\pm0.0799}$ 
        & $0.3048_{\pm0.0198}$ 
        & $0.3668_{\pm0.0632}$ 
        & $0.6016_{\pm0.0495}$ 
        & $0.6904_{\pm0.0413}$ 
        \\ 
        \texttt{GCN+DeBERTa} 
        & $0.2810_{\pm0.0136}$
        & $0.3598_{\pm0.0655}$
        & $0.4744_{\pm0.1281}$
        & $0.6482_{\pm0.0268}$
        & $0.6940_{\pm0.0286}$
        \\
        {\texttt{SFGL(ours)}}         
        & \textcolor{darkblue}{${0.6086}_{\pm0.0140}$} 
        & \textcolor{darkblue}{${0.6314}_{\pm0.0191}$} 
        & \textcolor{black}{$0.6612_{\pm0.0206}$} 
        & \textcolor{darkblue}{${0.6942}_{\pm0.0101}$} 
        & \textcolor{darkblue}{${0.7204}_{\pm0.0153}$} 
        \\
        \texttt{DeBERTa+GPT} 
        & $0.2776_{\pm0.0128}$ 
        & $0.5450_{\pm0.1121}$ 
        & \textcolor{darkblue}{${0.6624}_{\pm0.0067}$} 
        & $0.6888_{\pm0.0189}$ 
        & $0.7082_{\pm0.0184}$ 
        \\ 
        \texttt{GCN+DeBERTa+GPT} 
        & $0.4862_{\pm0.0489}$
        & $0.6272_{\pm0.0217}$
        & $0.6550_{\pm0.0077}$
        & $0.6836_{\pm0.0193}$
        & $0.6978_{\pm0.0085}$
        \\ 
        {\texttt{SFGL+GPT(ours)}} 
        & \textcolor{darkred}{${0.6348}_{\pm0.0136}$} 
        & \textcolor{darkred}{${0.6856}_{\pm0.0131}$} 
        & \textcolor{darkred}{${0.6992}_{\pm0.0265}$} 
        & \textcolor{darkred}{${0.7278}_{\pm0.0041}$} 
        & \textcolor{darkred}{${0.7374}_{\pm0.0072}$}
        \\
        \midrule
        \texttt{Pubmed}
        & $0.30\%_{(60)}$ & $0.61\%_{(120)}$ & $0.91\%_{(180)}$ & $1.83\%_{(360)}$ & $2.54\%_{(500)}$ 
        \\ 
        \midrule
        \texttt{GCN} 
        & $0.6770_{\pm0.0217}$ 
        & $0.7022_{\pm0.0107}$ 
        & $0.7398_{\pm0.0096}$ 
        & $0.7388_{\pm0.0041}$ 
        & $0.7464_{\pm0.0100}$ 
        \\ 
        \texttt{DeBERTa} 
        & $0.3616_{\pm0.0976}$
        & $0.4114_{\pm0.0532}$ 
        & $0.4402_{\pm0.0571}$ 
        & $0.5732_{\pm0.0363}$ 
        & $0.7474_{\pm0.0832}$ 
        \\ 
        \texttt{GCN+DeBERTa} 
        & $0.4702_{\pm0.0112}$
        & $0.5046_{\pm0.0180}$
        & $0.5520_{\pm0.0477}$
        & $0.6648_{\pm0.0453}$
        & $0.8042_{\pm0.0737}$        
        \\
        {\texttt{SFGL(ours)}}         
        & \textcolor{darkblue}{${0.7238}_{\pm0.0124}$} 
        & \textcolor{darkblue}{${0.7792}_{\pm0.0263}$} 
        & \textcolor{darkblue}{${0.8188}_{\pm0.0190}$} 
        & \textcolor{black}{$0.8524_{\pm0.0280}$} 
        & \textcolor{black}{$0.8390_{\pm0.0204}$} 
        \\
        \texttt{DeBERTa+GPT} 
        & $0.4514_{\pm0.0700}$ 
        & $0.5206_{\pm0.1020}$ 
        & $0.6196_{\pm0.1403}$ 
        & \textcolor{darkred}{${0.9278}_{\pm0.0034}$}
        & \textcolor{darkblue}{${0.9274}_{\pm0.0032}$}
        \\ 
        \texttt{GCN+DeBERTa+GPT} 
        & \textcolor{black}{$0.7386_{\pm0.0129}$}
        & $0.7440_{\pm0.0272}$
        & $0.7868_{\pm0.0212}$
        & \textcolor{darkblue}{${0.9242}_{\pm0.0049}$}
        & \textcolor{darkred}{${0.9282}_{\pm0.0022}$}
        \\ 
        {\texttt{SFGL+GPT(ours)}} 
        & \textcolor{darkred}{${0.8640}_{\pm0.0381}$} 
        & \textcolor{darkred}{${0.9014}_{\pm0.0064}$} 
        & \textcolor{darkred}{${0.9088}_{\pm0.0037}$} 
        & \textcolor{black}{$0.9162_{\pm0.0059}$} 
        & \textcolor{black}{$0.9116_{\pm0.0027}$}
        \\ 
        \midrule
        \texttt{ogbn-arxiv}
         & $0.03\%_{(50)}$ & $0.06\%_{(100)}$ & $0.09\%_{(150)}$ & $0.18\%_{(300)}$ & $0.30\%_{(500)}$ 
        \\ 
        \midrule
        \texttt{GCN} 
        & \textcolor{darkblue}{${0.3375}_{\pm0.0360}$} 
        & $0.3578_{\pm0.0179}$ 
        & $0.3991_{\pm0.0214}$ 
        & $0.4227_{\pm0.0132}$ 
        & $0.4516_{\pm0.0087}$  
        \\ 
        \texttt{DeBERTa} 
        & $---$ 
        & $---$ 
        & $---$ 
        & $---$ 
        & $0.2541_{\pm0.0636}$  
        \\ 
        \texttt{GCN+DeBERTa} 
        & $0.1661_{\pm0.0584}$
        & $0.1903_{\pm0.0523}$
        & $0.1964_{\pm0.0464}$
        & $0.2130_{\pm0.0169}$
        & $0.2359_{\pm0.0240}$
        \\
        {\texttt{SFGL(ours)}}         
        & $0.3251_{\pm0.0373}$ 
        & \textcolor{darkblue}{${0.4171}_{\pm0.0275}$} 
        & \textcolor{darkblue}{${0.4626}_{\pm0.0379}$} 
        & \textcolor{darkblue}{${0.4979}_{\pm0.0238}$} 
        & \textcolor{darkblue}{${0.5577}_{\pm0.0060}$}  
        \\ 
        \texttt{DeBERTa+GPT} 
        & $---$ 
        & $---$ 
        & $---$ 
        & $0.1653_{\pm0.0575}$ 
        & $0.2892_{\pm0.0182}$  
        \\      
        \texttt{GCN+DeBERTa+GPT} 
        & $0.2137_{\pm0.0317}$
        & $0.2200_{\pm0.0449}$
        & $0.2188_{\pm0.0709}$
        & $0.2321_{\pm0.0929}$
        & $0.2644_{\pm0.0185}$ 
        \\ 
        {\texttt{SFGL+GPT(ours)}} 
        & \textcolor{darkred}{${0.4570}_{\pm0.0625}$} 
        & \textcolor{darkred}{${0.6007}_{\pm0.0172}$} 
        & \textcolor{darkred}{${0.6380}_{\pm0.0159}$} 
        & \textcolor{darkred}{${0.6561}_{\pm0.0145}$} 
        & \textcolor{darkred}{${0.6567}_{\pm0.0182}$}  
        \\ 
        \midrule
         \texttt{arxiv23} 
         & $0.11\%_{(50)}$ & $0.22\%_{(100)}$ & $0.32\%_{(150)}$ & $0.65\%_{(300)}$ & $1.08\%_{(500)}$ 
         \\
         \midrule
        \texttt{GCN} 
        & $0.3809_{\pm0.0251}$
        & $0.4229_{\pm0.0137}$ 
        & $0.4342_{\pm0.0092}$
        & $0.4771_{\pm0.0076}$
        & $0.4976_{\pm0.0093}$
        \\ 
        \texttt{DeBERTa} 
        & $---$
        & $---$ 
        & $0.2204_{\pm0.0148}$
        & $0.2972_{\pm0.0339}$
        & $0.5619_{\pm0.0198}$
        \\ 
        \texttt{GCN+DeBERTa} 
        & $0.2506_{\pm0.0146}$
        & $0.2659_{\pm0.0102}$ 
        & $0.2648_{\pm0.0067}$
        & $0.2803_{\pm0.0040}$
        & $0.2976_{\pm0.0163}$
        \\
        {\texttt{SFGL(ours)}}     
        & \textcolor{darkblue}{${0.4323}_{\pm0.0259}$} 
        & \textcolor{darkblue}{${0.5002}_{\pm0.0125}$} 
        & \textcolor{darkblue}{${0.5160}_{\pm0.0127}$}
        & \textcolor{darkblue}{${0.5699}_{\pm0.0036}$}
        & \textcolor{black}{$0.5973_{\pm0.0080}$}
        \\
        \texttt{DeBERTa+GPT} 
        & $---$
        & $---$ 
        & $0.2223_{\pm0.0156}$
        & $0.3967_{\pm0.0417}$
        & \textcolor{darkblue}{${0.6137}_{\pm0.0497}$}
        \\ 
        \texttt{GCN+DeBERTa+GPT} 
        & $0.3261_{\pm0.0389}$  
        & $0.3474_{\pm0.0404}$ 
        & $0.3851_{\pm0.0123}$
        & $0.3974_{\pm0.0307}$
        & $0.4686_{\pm0.0199}$
        \\
        {\texttt{SFGL+GPT(ours)}}   
        & \textcolor{darkred}{${0.4953}_{\pm0.0066}$}  
        & \textcolor{darkred}{${0.5811}_{\pm0.0167}$} 
        & \textcolor{darkred}{${0.6041}_{\pm0.0164}$}
        & \textcolor{darkred}{${0.6439}_{\pm0.0133}$}
        & \textcolor{darkred}{${0.6552}_{\pm0.0247}$}
        \\ 
        \bottomrule
    \end{tabular}}
       \vspace{-2mm}
    \label{mainresults}
\end{table*}

\subsection{Experimental Settings}
\label{Experiment_setting}
\textbf{Datasets}. We conduct extensive experiments on four citation networks: \texttt{Cora} \citep{Cora}, \texttt{Pubmed} \citep{kipf2016semi}, \texttt{ogbn-arxiv} \citep{arxiv}, and \texttt{arxiv23} \citep{he2023harnessing} (see Appendix \ref{datasetdescription} for more details.).
Following standard practice, we utilize traditional text encoders to process textual attributes for graph generation. 
For \texttt{Cora}, each publication is represented using a bag-of-words model~\citep{bow} with a vocabulary of $1,433$ unique words. For \texttt{Pubmed}, documents are encoded as TF/IDF weighted word vectors~\citep{AIZAWA200345} from a dictionary of $500$ unique words. For \texttt{ogbn-arxiv} and \texttt{arxiv23}, textual features are extracted using a skip-gram model \citep{skipgram} and a word2vec model \citep{skipgram}, respectively. 
Additional comparisons with other text encoders are provided in Sec. \ref{ablation}.
In our experiments, the learning models have access to node attributes, while edge attributes remain unavailable. 

\textbf{Baselines}.
We compare our approach against several state-of-the-art methods.
\circlednum{1} \texttt{GCN} \citep{kipf2016semi}: trains a GCN using hand-crafted shallow features $\{\mathbf{f}_i\}^{m+n}$ and labels $\{\mathbf{y}_i\}^m$.
\circlednum{2} \texttt{DeBERTa} \citep{he2021deberta}: finetunes a pretrained \texttt{DeBERTa} on raw text attributes $\{\mathbf{x}_i\}^m$ and labels $\{\mathbf{y}_i\}^m$.
\circlednum{3} \texttt{GCN+DeBERTa} \citep{GLEM}: trains a GCN using text embeddings $\{\mathbf{e}_i\}^{m+n}$ and labels $\{\mathbf{y}_i\}^m$, where $\{\mathbf{e}_i\}^{m+n}$ are obtained by finetuning \texttt{DeBERTa} on $\{\mathbf{x}_i, \mathbf{y}_i\}^m$. 
\circlednum{4} \texttt{DeBERTa+GPT} \citep{chatgpt}: finetunes a \texttt{DeBERTa} using \texttt{GPT}-generated responses $\{\mathbf{r}_i\}^m$ of labeled samples and labels $\{\mathbf{y}_i\}^m$.
\circlednum{5} \texttt{GCN+DeBERTa+GPT} \citep{he2023harnessing}: trains a GCN using text embeddings $\{\mathbf{e}_i\}^{m+n}$ and labels $\{\mathbf{y}_i\}^m$, where $\{\mathbf{e}_i\}^{m+n}$ are obtained by finetuning a \texttt{DeBERTa} on $\{\mathbf{r}_i, \mathbf{y}_i\}^m$.
\circlednum{6} \texttt{SFGL} (ours): trains a GCN using text embeddings $\{\mathbf{e}_i\}^{m+n}$ and labels $\{\mathbf{y}_i\}^m$, where $\{\mathbf{e}_i\}^{m+n}$ are obtained by finetuning a \texttt{DeBERTa} with pseudo-labels $\{\hat{\mathbf{y}}_j\}^n$. 
\circlednum{7} \texttt{SFGL+GPT} (ours): similar to \texttt{SFGL} but uses \texttt{GPT}-generated responses to finetune the \texttt{DeBERTa}.

\subsection{Comparison Analysis}
Table~\ref{mainresults} presents the classification performance of various methods across different datasets and labeling rates. 
Our key observations are as follows:
\circlednum{1} Across all datasets, classification accuracy generally improves as the labeling rate increases. This aligns with expectations, as more labeled data provides better supervision, enhancing model performance.
\circlednum{2} The baselines \texttt{DeBERTa} and \texttt{DeBERTa+GPT} perform poorly at low labeling rates, particularly on \texttt{ogbn-arxiv} and \texttt{arxiv23}, where finetuning LMs becomes  unreliable under extreme label scarcity. 
This is because insufficient labeled data fails to provide sufficient information for gradient updates, leading to the learning of overly narrow patterns, which impairs generalization.
However, incorporating additional pseudo-labels enables LMs to generate informative text embeddings, thereby enhancing performance. 
\circlednum{3}  Incorporating GPT-generated responses significantly enhances both \texttt{GCN} and \texttt{DeBERTa}, particularly on the \texttt{Pubmed} dataset. This highlights the potential of LLMs to extract and leverage rich semantic information for text modeling.
\circlednum{4} In most cases, the proposed \texttt{SFGL} and \texttt{SFGL+GPT} achieve the highest performance, demonstrating their effectiveness in low-supervision scenarios.

\begin{table*}[!tb]
    \footnotesize
    \caption{Performance comparison of the proposed method with and without iterative training. }
    \centering
    \setlength{\tabcolsep}{1.35mm}{
    \begin{tabular}{l|ccccc}
        \toprule
        Labeling Rates (labels)
        & $0.30\%_{(60)}$ & $0.61\%_{(120)}$ & $0.91\%_{(180)}$ & $1.83\%_{(360)}$ & $2.54\%_{(500)}$ 
        \\ 
        \midrule
        \texttt{SFGL} 
        & \textcolor{black}{$0.7238_{\pm0.0124}$} 
        & \textcolor{black}{$0.7792_{\pm0.0263}$} 
        & \textcolor{black}{$0.8188_{\pm0.0190}$} 
        & \textcolor{black}{$0.8524_{\pm0.0280}$} 
        & \textcolor{black}{$0.8390_{\pm0.0204}$}
        \\ 
        \texttt{SFGL+\underline{GCN}}
        & \textcolor{black}{$0.7434_{\pm0.0048}$} 
        & \textcolor{black}{$0.8510_{\pm0.0096}$}  
        & \textcolor{black}{$0.8686_{\pm0.0263}$} 
        & \textcolor{black}{$0.8976_{\pm0.0084}$} 
        & \textcolor{black}{$0.8966_{\pm0.0070}$} 
        \\ 
        \texttt{SFGL+\uwave{GCN}}
        & \textcolor{black}{$0.8912_{\pm0.0293}$} 
        & \textcolor{black}{$0.8974_{\pm0.0116}$}  
        & \textcolor{black}{$0.9110_{\pm0.0083}$} 
        & \textcolor{black}{$0.9066_{\pm0.0094}$} 
        & \textcolor{black}{$0.9186_{\pm0.0073}$} 
        \\ 
        \texttt{SFGL+GPT}
        & \textcolor{black}{$0.8640_{\pm0.0381}$} 
        & \textcolor{black}{$0.9014_{\pm0.0064}$} 
        & \textcolor{black}{$0.9088_{\pm0.0037}$} 
        & \textcolor{black}{$0.9162_{\pm0.0059}$} 
        & \textcolor{black}{$0.9116_{\pm0.0027}$} 
        \\ 
        \texttt{SFGL+GPT+\underline{GCN}}
        & \textcolor{darkred}{${0.9142}_{\pm0.0086}$} 
        & \textcolor{darkred}{${0.9250}_{\pm0.0043}$}  
        & \textcolor{darkred}{${0.9294}_{\pm0.0040}$} 
        & \textcolor{darkred}{${0.9290}_{\pm0.0032}$} 
        & \textcolor{darkred}{${0.9374}_{\pm0.0015}$} 
        \\ 
        \texttt{SFGL+GPT+\uwave{GCN}}
        & \textcolor{darkblue}{${0.9138}_{\pm0.0029}$} 
        & \textcolor{darkblue}{${0.9198}_{\pm0.0054}$}  
        & \textcolor{darkblue}{${0.9216}_{\pm0.0080}$} 
        & \textcolor{darkblue}{${0.9240}_{\pm0.0046}$} 
        & \textcolor{darkblue}{${0.9264}_{\pm0.0026}$} 
        \\ 
        \bottomrule
    \end{tabular}}
    \label{iterative_performance}
       \vspace{-2mm}
\end{table*}

\begin{table*}[!tb]
    \footnotesize
    \centering
    \caption{Performance comparison of different models using various distance metrics, including Manhattan (\texttt{M}), Euclidean (\texttt{E}), and Cosine (\texttt{C}) distances.}
    \setlength{\tabcolsep}{0.95mm}{ 
    \begin{tabular}{l|ccccc}
        \toprule
         {Labeling Rates (labels)}
         & $2.58\%_{(70)}$ & $5.17\%_{(140)}$ & $7.75\%_{(210)}$ & $15.51\%_{(420)}$ & $22.16\%_{(600)}$
        \\ 
        \midrule
        {\texttt{SFGL(M)}}         
        & \textcolor{black}{${0.5454}_{\pm0.0327}$} 
        & \textcolor{black}{${0.5528}_{\pm0.1459}$} 
        & \textcolor{black}{$0.6206_{\pm0.0101}$} 
        & \textcolor{black}{${0.6810}_{\pm0.0080}$} 
        & \textcolor{black}{${0.7144}_{\pm0.0109}$} 
        \\ 
        {\texttt{SFGL(E)}}         
        & \textcolor{black}{${0.5178}_{\pm0.0271}$} 
        & \textcolor{black}{${0.5524}_{\pm0.1468}$} 
        & \textcolor{black}{$0.6240_{\pm0.0163}$} 
        & \textcolor{black}{${0.6718}_{\pm0.0136}$} 
        & \textcolor{black}{${0.7138}_{\pm0.0141}$} 
        \\ 
        {\texttt{SFGL(C)}}         
        & \textcolor{darkblue}{${0.6086}_{\pm0.0140}$} 
        & \textcolor{darkblue}{${0.6314}_{\pm0.0191}$} 
        & \textcolor{darkblue}{$0.6612_{\pm0.0206}$} 
        & \textcolor{darkblue}{${0.6942}_{\pm0.0101}$} 
        & \textcolor{darkblue}{${0.7204}_{\pm0.0153}$} 
        \\ 
        \cdashline{1-6}
        {\texttt{SFGL+GPT(M)}}         
        & \textcolor{black}{${0.6176}_{\pm0.0291}$} 
        & \textcolor{black}{${0.6564}_{\pm0.0223}$} 
        & \textcolor{black}{$0.6666_{\pm0.0229}$} 
        & \textcolor{black}{${0.6954}_{\pm0.0215}$} 
        & \textcolor{black}{${0.7214}_{\pm0.0144}$} 
        \\ 
        {\texttt{SFGL+GPT(E)}}         
        & \textcolor{black}{${0.6138}_{\pm0.0168}$} 
        & \textcolor{black}{${0.6474}_{\pm0.0132}$} 
        & \textcolor{black}{$0.6588_{\pm0.0110}$} 
        & \textcolor{black}{${0.6932}_{\pm0.0249}$} 
        & \textcolor{black}{${0.7238}_{\pm0.0098}$} 
        \\
        {\texttt{SFGL+GPT(C)}} 
        & \textcolor{darkred}{${0.6348}_{\pm0.0136}$} 
        & \textcolor{darkred}{${0.6856}_{\pm0.0131}$} 
        & \textcolor{darkred}{${0.6992}_{\pm0.0265}$} 
        & \textcolor{darkred}{${0.7278}_{\pm0.0041}$} 
        & \textcolor{darkred}{${0.7374}_{\pm0.0072}$}
        \\ 
        \midrule
        {\texttt{DeBERTa+GCN(M)}}         
        & \textcolor{black}{${0.4956}_{\pm0.0136}$} 
        & \textcolor{black}{${0.5286}_{\pm0.1312}$} 
        & \textcolor{black}{$0.6134_{\pm0.0101}$} 
        & \textcolor{black}{${0.6922}_{\pm0.0101}$} 
        & \textcolor{darkblue}{${0.7282}_{\pm0.0080}$} 
        \\ 
        {\texttt{DeBERTa+GCN(E)}}         
        & \textcolor{black}{${0.4660}_{\pm0.0261}$} 
        & \textcolor{black}{${0.5126}_{\pm0.0182}$} 
        & \textcolor{black}{$0.6194_{\pm0.0109}$} 
        & \textcolor{black}{${0.6882}_{\pm0.0087}$} 
        & \textcolor{black}{${0.7254}_{\pm0.0074}$}  
        \\ 
        {\texttt{DeBERTa+GCN(C)}}
        &\textcolor{darkblue}{$0.6232_{\pm0.0094}$}
        &\textcolor{darkblue}{$0.6616_{\pm0.0077}$}
        &\textcolor{darkblue}{$0.6728_{\pm0.0079}$}
        &\textcolor{darkblue}{$0.7064_{\pm0.0077}$}
        &\textcolor{black}{$0.7214_{\pm0.0096}$}
        \\
        \cdashline{1-6}
        {\texttt{DeBERTa+GPT+GCN(M)}}         
        & \textcolor{black}{${0.5380}_{\pm0.0135}$} 
        & \textcolor{black}{${0.6480}_{\pm0.0034}$} 
        & \textcolor{black}{$0.6656_{\pm0.0122}$} 
        & \textcolor{black}{${0.7092}_{\pm0.0088}$} 
        & \textcolor{black}{${0.7332}_{\pm0.0109}$}  
        \\ 
        {\texttt{DeBERTa+GPT+GCN(E)}}         
        & \textcolor{black}{${0.5268}_{\pm0.0146}$} 
        & \textcolor{black}{${0.6412}_{\pm0.0110}$} 
        & \textcolor{black}{$0.6634_{\pm0.0077}$} 
        & \textcolor{black}{${0.7112}_{\pm0.0131}$} 
        & \textcolor{black}{${0.7318}_{\pm0.0110}$}  
        \\ 
        {\texttt{DeBERTa+GPT+GCN(C)}}   
        & \textcolor{darkred}{$0.6588_{\pm0.0113}$}
        & \textcolor{darkred}{$0.6978_{\pm0.0200}$}
        & \textcolor{darkred}{$0.6962_{\pm0.0079}$}
        & \textcolor{darkred}{$0.7206_{\pm0.0109}$}
        & \textcolor{darkred}{$0.7426_{\pm0.0072}$}
        \\
        \bottomrule
    \end{tabular}}
       \vspace{-1mm}
    \label{distance_metric}
\end{table*}

\begin{table*}[!tb]
    \footnotesize
    \caption{Performance comparison between real graphs and KNN graphs.}
    \centering
    \setlength{\tabcolsep}{1.4mm}
    \begin{tabular}{l|ccccc}
    \toprule
     \multicolumn{1}{c|}{Labeling Rates (labels)} 
     & $0.11\%_{(50)}$ & $0.22\%_{(100)}$ & $0.32\%_{(150)}$ & $0.65\%_{(300)}$ & $1.08\%_{(500)}$ \\
        \midrule
        \texttt{Real Graph} 
        & $0.3691_{\pm0.0033}$ 
        & \textcolor{darkblue}{${0.4124}_{\pm0.0185}$} 
        & \textcolor{darkred}{${0.4354}_{\pm0.0107}$} 
        & \textcolor{black}{${0.4764}_{\pm0.0116}$} 
        & \textcolor{darkred}{${0.5039}_{\pm0.0071}$}   
        \\ 
        \texttt{KNN Graph k=10}
        & $0.3660_{\pm0.0136}$ 
        & $0.4031_{\pm0.0054}$ 
        & $0.4152_{\pm0.0145}$ 
        & $0.4541_{\pm0.0114}$ 
        & $0.4706_{\pm0.0164}$   
        \\ 
         \texttt{KNN Graph k=20}
        & \textcolor{darkblue}{${0.3809}_{\pm0.0251}$} 
        & \textcolor{darkred}{${0.4229}_{\pm0.0137}$} 
        & \textcolor{darkblue}{$0.4342_{\pm0.0092}$} 
        & \textcolor{darkblue}{$0.4771_{\pm0.0076}$} 
        & $0.4976_{\pm0.0093}$   
        \\ 
        \texttt{KNN Graph k=30}
        & $0.3795_{\pm0.0278}$ 
        & $0.4085_{\pm0.0026}$ 
        & $0.4130_{\pm0.0099}$ 
        & $0.4595_{\pm0.0149}$ 
        & $0.4864_{\pm0.0097}$   
        \\ 
       \texttt{KNN Graph k=50}
        & \textcolor{darkred}{${0.3818}_{\pm0.0149}$} 
        & $0.4081_{\pm0.0081}$ 
        & \textcolor{black}{${0.4203}_{\pm0.0133}$}
        & $0.4679_{\pm0.0194}$ 
        & $0.4857_{\pm0.0207}$   
        \\ 
         \texttt{KNN Graph k=100}
        & $0.3767_{\pm0.0182}$ 
        & $0.4100_{\pm0.0055}$ 
        & $0.4201_{\pm0.0091}$ 
        & \textcolor{darkred}{${0.4778}_{\pm0.0097}$} 
        & \textcolor{darkblue}{${0.5005}_{\pm0.0119}$}  
        \\ 
        \bottomrule
    \end{tabular}
   \vspace{-1mm}
\label{table4}
\end{table*}

\subsection{Ablation Study}
\label{ablation}
\textbf{Iterative Performance}.
As mentioned in Sec. \ref{final_class}, improving the GNNs allows us to generate more accurate pseudo-labels.
These enhanced pseudo-labels, in turn, enable us to retrain the LMs, producing more refined text embeddings (see Appendix \ref{trainingstrategies} for detailed training stages).
To evaluate this iterative training strategy, we introduce two improved GNN variants, \texttt{\underline{GCN}} and \texttt{\uwave{GCN}}, derived from \texttt{SFGL} and \texttt{SFGL+GPT}, respectively, as new pseudo-labelers. 
Table~\ref{iterative_performance} presents the iterative performance of our models on the \texttt{Pubmed} dataset.
Our results demonstrate that this iterative mechanism brings further performance improvements, highlighting the mutually beneficial relationship between GNNs and LMs in iterative training.
Notably, \texttt{SFGL+GPT+\underline{GCN}} consistently outperforms \texttt{SFGL+GPT+\uwave{GCN}}, indicating slight overfitting during the iterative training process.

\textbf{Different Distance Metrics}.
As discussed in Sec. \ref{KNN}, various distance metrics, such as Euclidean, Manhattan, and Cosine distances, can be used in KNN graphs to measure similarity between samples.
To evaluate their impact, we compare model performance using these different metrics. 
Table \ref{distance_metric} presents the results for \texttt{SFGL}, \texttt{SFGL+GPT}, \texttt{DeBERTa+GCN}, and \texttt{DeBERTa+GPT+GCN}. 
The first two models use GCNs as the final classifier, while the latter two employ LMs as the classifier but leverage our graph-based pseudo-labeler to enhance LM finetuning. 
As shown, Cosine distance consistently outperforms the other metrics, particularly at very low labeling rates, further justifying its effectiveness for graph construction in our framework.

\textbf{Real Graph vs. KNN Graph}.
Table~\ref{table4} presents the performance comparison between real graphs and KNN graphs on the \texttt{arxiv23} dataset, considering different labeling rates and values of $k$.
Notably, under extremely low labeling rates, the performance achieved with the true graph is not significantly better—and in some cases, even worse—than that obtained with a pre-constructed KNN graph. 
This suggests that a true graph may not be crucial for improving LM finetuning, particularly in scenarios with very limited supervision.
Our explanation for this phenomenon is that both the number of labeled nodes and the choice of which nodes to label play a crucial role in determining the final classification performance. Consider an extreme case: if all labeled nodes in a real graph are isolated, label information cannot propagate to the unlabeled nodes. In contrast, constructing an artificial graph for these nodes may establish connections between the unlabeled nodes and the isolated labeled nodes, facilitating label propagation across the graph. Consequently,  artificial graphs can potentially outperform real graphs, particularly in scenarios with very limited supervision. 
See Appendix \ref{moreresults} for additional experimental results.

\section{Related Work}

\textbf{Latent Graph Inference Models}. 
Traditional LGI methods adopt linear projections to learn node representations and optimize various objective functions to learn latent graphs.
For example, \cite{zhang2010graph} design an entropy regularization to control the uniformity level of edge weights. \cite{nie2016unsupervised} infer an optimal graph by assigning adaptive neighbors. \cite{lu2018subspace} impose spectral sparsity on the graph Laplacian matrix.
\cite{lu2021target} assume that two samples are likely to belong to the same class if one sample is close to the reconstructed representation of the other.  
Advanced LGI models exploit GNNs to learn the latent graphs. 
For instance, \cite{LDS} model a discrete probability distribution for the edges. 
\cite{SLAPS} provide supplementary supervision for latent graphs through a self-supervision task. 
\cite{jianglin2023LGI} propose a model-agnostic model that obtains supplementary supervision directly from true labels.
However, existing methods typically rely on artificial assumptions about the underlying edge distribution.
For example, \cite{zhang2010graph} impose a uniformity assumption on edge weights.
\cite{lu2018subspace} introduce a block diagonal prior to the graph Laplacian matrix.
\cite{LU2021107758} construct an adaptive neighborhood graph by assuming the probability property of edge weights.
\cite{SLAPS} assume that a graph structure effective for predicting features is also effective for label prediction.
{These artificial assumptions may not accurately reflect real graph structures and require specific optimizations with additional model training across the entire dataset.}

\textbf{Language-Assisted Graph Models}.
LGI models typically rely on feature engineering approaches, such as skip-gram and TF-IDF, to encode textual sequences into feature vectors. 
In contrast, recent LAG models seek to enhance text embeddings by leveraging various LMs to extract richer semantic features from text sequences \citep{li2024surveygraphmeetslarge}.
For example, \cite{GLEM} design a variational expectation-maximization framework to fuse graph structure and language learning for classification.
\cite{duan2023simteg} first conduct parameter-efficient finetuning on a pretrained LM and then generate text embeddings using the finetuned LM. 
\cite{he2023harnessing} leverages an LLM to capture textual information and applies a small LM as the interpreter to transform the responses of the LLM into informative features. 
\cite{yu2023empower} use LLMs to generate nodes and edges for text-attributed graphs, which harnesses LLMs for enhancing class-level information.
For most existing LAG models, finetuning an LM on the target dataset is essential to generate semantically enriched text embeddings. 
However, it is notorious that finetuning a pretrained LM typically demands a large amount of annotated data \citep{GPT3}, which poses a significant challenge for LAG models in semi-supervised learning tasks, where available annotated data is often scarce.

\section{Conclusion}
In this paper, we identify two primary challenges in exiting graph-language models for semi-supervised learning, i.e., artificial structural assumptions in graph generation and unreliable LM finetuning for text embedding. 
We tackle these challenges by establishing a well-grounded structural prior. Specifically, we examine the scale-free property of citation networks and reveal that this structural characteristic can be effectively approximated using a simple KNN graph. Building on this observation, we propose a novel graph-language model that employs a customized KNN algorithm for scale-free graph generation and utilizes a graph-based pseudo-labeler to provide additional supervision for improved text embedding. 
Extensive experiments on representative datasets validate our findings on the scale-free structural approximation of KNN graphs and demonstrate the effectiveness of integrating graph generation and text embedding with a real structural prior.
We hope this study highlights the synergistic potential of GNNs and LMs, providing valuable insights for researchers in both the GNN and NLP communities.

\clearpage

\normalem
\bibliography{iclr2025_conference}
\bibliographystyle{iclr2025_conference}

\clearpage

\appendix
\section{Appendix}

\subsection{Dataset Description}
\label{datasetdescription}
We conducted extensive experiments on four text-attributed graph benchmarks: \texttt{Cora}~\citep{Cora}, \texttt{PubMed}~\citep{kipf2016semi}, \texttt{ogbn-arxiv}~\citep{arxiv}, and \texttt{arxiv23}~\citep{he2023harnessing}.
The \texttt{Cora} dataset contains $2,708$ papers belonging to seven different classes.
The \texttt{PubMed} dataset includes $19,717$ publications categorized into three classes.
The \texttt{ogbn-arxiv}~\citep{arxiv} dataset includes $169,343$ arXiv papers across $40$ subject areas in computer science. Labels for this dataset are manually annotated by the authors and arXiv.
The \texttt{arxiv23} \citep{he2023harnessing} dataset includes all computer science papers submitted in $2023$ or later on arXiv. Similar to ogbn-arxiv, these papers are divided into $40$ classes.

\subsection{Model Configuration}
\label{modelconfiguration}
We follow~\citep{he2023harnessing} for model configurations.
For GNNs, we use a two-layer GCN (see Sec. \ref{ablation} for comparisons with other GNN architectures). The hyper-parameters for hidden dimension, learning rate, dropout ratio, and weight decay are set to $128$, $0.001$, $0.5$, and $0.0005$, respectively.
For LMs, we finetune a pretrained DeBERTa~\citep{he2021deberta} on the target datasets. The batch size, learning rate, and dropout ratio are set to $20$, $2 \times 10^{-5}$, and $0.3$, respectively.
We also employ GPT3.5~\citep{chatgpt} for inference.
We empirically set $k$ to $25, 15, 25, 20$ for \texttt{Cora}, \texttt{Pubmed}, \texttt{ogbn-arxiv}, and \texttt{arxiv23}, respectively.
We conduct five independent experiments with different random seeds and report the average test accuracy along with the standard deviation.

For LLM  prompt inputs, we follow established methods~\citep{he2023harnessing}. Each prompt consists of a paper's title and abstract, along with a task-specific question designed to elicit a class prediction and a corresponding explanation from the LLM.
The generated textual responses by the LLMs  serve as inputs for the LMs, which use them to generate text embeddings for all nodes. These embeddings are then further processed by GNNs  to enhance model performance.

\subsection{Proof of Proposition}
\setcounter{proposition}{0}
\begin{proposition}
For a scale-free network, assuming that there is at most one node $v$ whose degree belongs to $[\theta_{max}, \infty)$, we have $\theta_{max} = \theta_{min} |\mathcal{V}|^{\frac{1}{\alpha-1}}$, where $\theta_{max}$ and $\theta_{min}$ refer to the maximum and minimal degree in the scale-free network.
\end{proposition}

\begin{proof}
We write the degree distribution $\mathbb{P}(\theta)$ of a scale free network as:
$\mathbb{P}(\theta) = \xi \theta^{-\alpha}$, where $\theta$ is the degree of a node, $\xi$ is a constant, and $\alpha$ is a scaling parameter of the distribution.
Since $\int \theta^{-\alpha} d\theta = \frac{\theta^{1-\alpha}}{1-\alpha} + constant$ and $\int_{\theta_{min}}^{\infty}\mathbb{P}(\theta)d\theta=1$, we have $\frac{1}{\xi} = \int_{\theta_{min}}^{\infty} \theta^{-\alpha}d\theta=[\frac{\theta^{1-\alpha}}{1-\alpha}]_{\theta_{min}}^{\infty}$. 
Since $\lim_{\theta\to \infty} \frac{\theta^{1-\alpha}}{1-\alpha}=0$ $(\alpha > 1)$, 
we have ${[\frac{\theta^{1-\alpha}}{1-\alpha}]}_{\theta_{min}}^{\infty} = \frac{\theta_{min}^{1-\alpha}}{\alpha-1}$. As a result, we obtain $\xi=(\alpha-1)\theta_{min}^{\alpha-1}$ and $\mathbb{P}(\theta) = (\alpha-1)\theta_{min}^{\alpha-1} \theta^{-\alpha}$.
Assume that there is at most one node whose degree belongs to $[\theta_{max}, \infty)$, where $\theta_{max}$ is the maximum degree in the network. We have $\int_{\theta_{max}}^{\infty}\mathbb{P}(\theta)d\theta=\int_{\theta_{max}}^{\infty}\xi \theta^{-\alpha}d\theta=\xi \frac{\theta_{max}^{1-\alpha}}{\alpha-1} =(\alpha-1)\theta_{min}^{\alpha-1} \frac{\theta_{max}^{1-\alpha}}{\alpha-1}=\frac{1}{|\mathcal{V}|}$, where $|\mathcal{V}|$ is the number of nodes. 
After simplification, we obtain $(\frac{\theta_{min}}{\theta_{max}})^{\alpha-1}=\frac{1}{|\mathcal{V}|}$, which leads to $\theta_{max} = \theta_{min} |\mathcal{V}|^{\frac{1}{\alpha-1}}$.
\end{proof}

\begin{proposition}
For a KNN graph, we have $\sum_{v\in \mathcal{V}}C_{ID}(v) = \sum_{v\in \mathcal{V}}C_D(v)  - \sum_{v\in \mathcal{V}}C_{OD}(v) = 2 |\mathcal{E}|  - k |\mathcal{V}| $, where $|\mathcal{E}| \leq k |\mathcal{V}|$.
\label{prop1}
\end{proposition}

\begin{proof}
In a KNN graph, we have $C_D(v)=C_{ID}(v)+C_{OD}(v)$, and $\sum_{v\in \mathcal{V}}C_D(v) = 2 |\mathcal{E}|$. For $\forall v \in \mathcal{V}$, $C_{OD}(v)=k$.
Considering that two nodes may regard each other as its nearest neighbors, the total number of edges is at most $k |\mathcal{V}|$, i.e., $|\mathcal{E}| \leq k |\mathcal{V}|$.
Based on the above points, we can derive: $\sum_{v\in \mathcal{V}}C_{ID}(v) = \sum_{v\in \mathcal{V}}C_D(v)  - \sum_{v\in \mathcal{V}}C_{OD}(v) = 2 |\mathcal{E}|  - k |\mathcal{V}| $, where $|\mathcal{E}| \leq k |\mathcal{V}|$.
\end{proof}

\subsection{Additional Experimental Results}
\label{moreresults}

\begin{table*}[!tb]
    \footnotesize
    \caption{Performance comparison of different models using different types of GNNs.}
    \centering
    \setlength{\tabcolsep}{1.3mm}{ 
    \begin{tabular}{l|ccccc}
        \toprule
         \multicolumn{1}{c|}{Labeling Rates (labels)} 
         & $2.58\%_{(70)}$ & $5.17\%_{(140)}$ & $7.75\%_{(210)}$ & $15.51\%_{(420)}$ & $22.16\%_{(600)}$
        \\ 
        \midrule      
        \texttt{GCN} 
        & $0.5898_{\pm0.0159}$ 
        & $0.6038_{\pm0.0116}$ 
        & $0.6358_{\pm0.0164}$ 
        & $0.6936_{\pm0.0055}$ 
        & $0.7114_{\pm0.0049}$ 
        \\ 
        \texttt{SAGE} 
        & $0.5660_{\pm0.0109}$ 
        & $0.5994_{\pm0.0143}$ 
        & $0.6326_{\pm0.0150}$  
        & $0.6910_{\pm0.0106}$ 
        & $0.7044_{\pm0.0153}$ 
        \\ 
        \texttt{GAT} 
        & $0.6050_{\pm0.0052}$ 
        & $0.6262_{\pm0.0074}$  
        & $0.6560_{\pm0.0075}$ 
        & $0.6852_{\pm0.0092}$ 
        & $0.6944_{\pm0.0174}$ 
        \\ 
        \cdashline{1-6}
        {\texttt{SFGL(GCN)}}         
        & \textcolor{black}{${0.6086}_{\pm0.0140}$} 
        & \textcolor{black}{${0.6314}_{\pm0.0191}$} 
        & \textcolor{black}{$0.6612_{\pm0.0206}$} 
        & \textcolor{black}{${0.6942}_{\pm0.0101}$} 
        & \textcolor{black}{${0.7204}_{\pm0.0153}$} 
        \\ 
        {\texttt{SFGL(SAGE)}} 
        & $0.5794_{\pm0.0103}$ 
        & $0.6502_{\pm0.0100}$ 
        & $0.6706_{\pm0.0072}$ 
        & $0.7056_{\pm0.0135}$ 
        & $0.7234_{\pm0.0150}$ 
        \\ 
        {\texttt{SFGL(GAT)}} 
        & $0.5974_{\pm0.0117}$ 
        & $0.6586_{\pm0.0099}$ 
        & $0.6998_{\pm0.0079}$ 
        & $0.7148_{\pm0.0200}$ 
        & $0.7196_{\pm0.0081}$ 
        \\ 
        {\texttt{SFGL+GPT(GCN)}} 
        & \textcolor{darkblue}{${0.6348}_{\pm0.0136}$} 
        & \textcolor{darkblue}{${0.6856}_{\pm0.0131}$} 
        & \textcolor{black}{${0.6992}_{\pm0.0265}$} 
        & \textcolor{darkred}{${0.7278}_{\pm0.0041}$} 
        & \textcolor{darkblue}{${0.7374}_{\pm0.0072}$}
        \\ 
        {\texttt{SFGL+GPT(SAGE)}} 
        & $0.6342_{\pm0.0153}$ 
        & $0.6806_{\pm0.0142}$ 
        & \textcolor{darkblue}{$0.7050_{\pm0.0080}$} 
        & \textcolor{darkblue}{$0.7276_{\pm0.0162}$} 
        & $0.7348_{\pm0.0110}$ 
        \\ 
        {\texttt{SFGL+GPT(GAT)}}
        & \textcolor{darkred}{$0.6390_{\pm0.0223}$} 
        & \textcolor{darkred}{$0.6954_{\pm0.0136}$} 
        & \textcolor{darkred}{$0.7128_{\pm0.0199}$} 
        & $0.7256_{\pm0.0114}$ 
        & \textcolor{darkred}{$0.7390_{\pm0.0076}$} 
        \\ 
        \midrule
        {\texttt{DeBERTa+GCN}}
        &\textcolor{black}{$0.6232_{\pm0.0094}$}
        &\textcolor{black}{$0.6616_{\pm0.0077}$}
        &\textcolor{black}{$0.6728_{\pm0.0079}$}
        &\textcolor{black}{$0.7064_{\pm0.0077}$}
        &\textcolor{black}{$0.7214_{\pm0.0096}$}
        \\
        {\texttt{DeBERTa+SAGE}}
        & $0.5882_{\pm0.0081}$ 
        & $0.6396_{\pm0.0100}$ 
        & $0.6652_{\pm0.0090}$ 
        & $0.7122_{\pm0.0074}$ 
        & $0.7224_{\pm0.0091}$ 
        \\
        {\texttt{DeBERTa+GAT}}
        & $0.6112_{\pm0.0057}$ 
        & $0.6530_{\pm0.0104}$ 
        & \textcolor{darkblue}{$0.6986_{\pm0.0128}$} 
        & $0.7074_{\pm0.0132}$ 
        & $0.7186_{\pm0.0065}$ 
        \\
        {\texttt{DeBERTa+GPT+GCN}}   
        & \textcolor{darkred}{$0.6588_{\pm0.0113}$}
        & \textcolor{darkred}{$0.6978_{\pm0.0200}$}
        & \textcolor{black}{$0.6962_{\pm0.0079}$}
        & \textcolor{darkred}{$0.7206_{\pm0.0109}$}
        & \textcolor{darkred}{$0.7426_{\pm0.0072}$}
        \\
        {\texttt{DeBERTa+GPT+SAGE}}
        & \textcolor{darkblue}{$0.6500_{\pm0.0113}$} 
        & \textcolor{darkblue}{$0.6918_{\pm0.0051}$} 
        & $0.6914_{\pm0.0122}$ 
        & \textcolor{darkblue}{$0.7166_{\pm0.0063}$} 
        & \textcolor{darkblue}{$0.7312_{\pm0.0060}$} 
        \\
        {\texttt{DeBERTa+GPT+GAT}}
        & $0.6482_{\pm0.0103}$ 
        & $0.6910_{\pm0.0033}$ 
        & \textcolor{darkred}{$0.7114_{\pm0.0066}$} 
        & $0.7118_{\pm0.0059}$ 
        & $0.7294_{\pm0.0051}$ 
        \\
        \bottomrule
    \end{tabular}}
       \vspace{-1mm}
    \label{GNN_Model}
\end{table*}

As mentioned in Sec. \ref{TFE}, various GNNs can be used to implement the pseudo-labeler. 
To assess their impact, we evaluate model performance using different GNNs, including \texttt{GCN}, \texttt{SAGE}, and \texttt{GAT}.
Table \ref{GNN_Model} presents the comparison results across these variants. 
Our findings indicate that the proposed models exhibit robustness to the choice of GNN, as they achieve similar performance improvements regardless of the specific architecture used.

In Sec. \ref{Experiment_setting}, our experiments employ traditional text encoders for graph construction and pseudo-labeling. 
In fact, a variety of text encoders can be utilized for this task. To evaluate their impact, we experiment with several text encoders, including a traditional TF-IDF encoder (\texttt{Enc1}), a pretrained DeBERTa (\texttt{Enc2}), a pretrained E5-Mistral-7B-Instruct (\texttt{Enc3}) \citep{wang2024improvingtextembeddingslarge},  a pretrained bge-en-icl (\texttt{Enc4}) \citep{li2025making}, and a finetuned DeBERTa (\texttt{Enc5}). 
Table \ref{text_encoders} presents the comparison results on the \texttt{PubMed} dataset.
\begin{wraptable}{r}{0.53\textwidth}
    \centering   
    \vspace{-2mm}
    \footnotesize
    \caption{Results across different text encoders.}
    \vspace{1mm}
    \begin{tabular}{l|ccccc}
        \toprule
        Labels  
        & $60$ & $120$ & $180$ & $360$ & $500$ \\
        \midrule
        \texttt{Enc1} 
        & $0.665$ & $0.696$ & $0.748$ & $0.749$ & $0.758$ \\
        \texttt{Enc2} 
        & $0.475$ & $0.518$ & $0.534$ & $0.538$ & $0.548$ \\
        \texttt{Enc3} 
        & $0.550$ & $0.614$ & $0.648$ & $0.686$ & $0.690$ \\
        \texttt{Enc4} 
        & $0.596$ & $0.636$ & $0.648$ & $0.652$ & $0.667$ \\
        \texttt{Enc5}  
        & $0.743$ & $0.815$ & $0.856$ & $0.877$ & $0.876$ \\
        \bottomrule
    \end{tabular}
    \label{text_encoders}
\end{wraptable}
As observed, pretrained models such as DeBERTa, E5-Mistral-7B-Instruct, and bge-en-icl fail to achieve satisfactory results. In contrast, the fine-tuned DeBERTa model demonstrates significant improvement, with TF-IDF performance falling between that of the pretrained and finetuned models.
However, finetuning DeBERTa requires a sufficient number of labels, which is why we initially use shallow text encoders to facilitate pseudo-label generation.
Notably, the two variants of our proposed SFGL listed in Table \ref{iterative_performance} use the finetuned DeBERTa as their text encoder. As shown, incorporating a more refined text encoder leads to further performance gains.
Additionally, both pretrained E5-Mistral-7B-Instruct and bge-en-icl outperform the pretrained DeBERTa model but fall short of the finetuned DeBERTa. It is reasonable to infer that a finetuned E5-Mistral-7B-Instruct or bge-en-icl model would surpass the finetuned DeBERTa in performance.

Figures \ref{in_degree_distribution_a1} and \ref{fitting_curves_a1} illustrate the degree distribution and fitting curves for $k=20$. As shown, the degree distributions of KNN graphs with a larger $k$ still approximate a scale-free structure. 
We further analyze the KNN graph on the \texttt{Cora} dataset using Euclidean distance as the similarity metric. As depicted in Fig. \ref{other_distribution} (b), the degree distribution of the KNN graph ($k=5$) constructed with Euclidean distance does not follow a power law pattern. Notably, $2,080$ nodes have zero in-degrees, meaning they are not among the top $5$ nearest neighbors of any nodes. This phenomenon likely arises from the way text embeddings encode semantic and syntactic information, which simple Euclidean distance fails to capture effectively. 
In other words, raw Euclidean distances between embedding vectors do not adequately reflect semantic similarities.
This is because Euclidean distance primarily measures linear separation between points while disregarding the angular relationships or vector orientation, which are key factors in capturing semantic meaning.

Besides KNN graphs, our training framework can incorporate any LGI method to generate pseudo-labels. However, unlike KNN graphs, existing LGI methods typically require specialized optimization strategies with complex constraints and additional training on the entire dataset, leading to increased computational and model complexity.
Moreover, ensuring that graphs generated by LGI methods exhibit a scale-free structure is challenging. 
For instance, SLAPS \citep{SLAPS} requires several hours of training to generate graphs, whereas KNN graphs can be constructed in just a few minutes. Additionally, as shown in Fig. \ref{other_distribution} (a), the graph produced by SLAPS does not exhibit scale-free properties, further highlighting the efficiency and structural advantages of KNN graphs.

Additionally, we evaluate the performance of our models using the full set of available labels on the \texttt{Cora} and \texttt{PubMed} datasets.
According to the standard data split, there are $1,068$ labeled samples in \texttt{Cora} and $18,157$ labeled samples in \texttt{PubMed} that remain unutilized during training. To assess the impact of complete supervision, we incorporate these additional labels, expanding the training sets to $1,208$ and $18,217$ labeled nodes for \texttt{Cora} and \texttt{PubMed}, respectively.
Table \ref{complete_lables} presents the results comparing the standard and complete label settings. With full supervision, the standalone DeBERTa model achieves strong performance, eliminating the need for pseudo-label generation, as the LMs already produce high-quality text embeddings. However, under the standard split with limited labeled nodes, generating pseudo-labels remains crucial for enhancing performance.

\begin{algorithm}[!tb]
\caption{Scale-Free Graph-Language Model (\texttt{SFGL} / \texttt{SFGL+GPT})}
\begin{algorithmic}[1]
\Require Input sequences $\mathbf{X}$, labels $\mathbf{Y}_L$, the number of neighbors $k$, query question $\mathbf{q}$, prompt $\mathbf{p}_i$.
\State Encode the text sequences of nodes into hand-crafted features: $\mathbf{F}$ = \texttt{Enc}($\mathbf{X}$).
\State Construct a KNN graph with Cosine distance metric: $\mathbf{A}=\texttt{KNN}_{\texttt{cos}}(\mathbf{F}, k)$.
\State Train a GNN ($\texttt{GNN}_{\mathbf{\Phi}}$) on $\mathbf{F}$, $\mathbf{A}$, and $\mathbf{Y}_L$ with loss function in Eq. \ref{gnn_train}. 
\State Generate pseudo labels for the unlabeled nodes:     $\hat{\mathbf{y}}_j = \operatorname{argmax}(\texttt{GNN}_{\mathbf{\Phi}}(\mathbf{f}_j, \mathbf{A}))$, $j \in \{1, \cdots, n\}$.
\State (Optional) Generate textual responses from an LLM : $\mathbf{r}_i = \texttt{LLM}(\mathbf{p}_i), \  \mathbf{p}_i=\operatorname{concat}(\mathbf{x}_i, \mathbf{q})$.
\State Finetune a LM $\texttt{LM}_{\mathbf{\Theta}}$ using loss function in Eq. \ref{lm_finetune}. 
\State Generate text embeddings: $\mathbf{e}_i=\texttt{LM}_{\mathbf{\Theta}}(\mathbf{x}_i)$ (If LLM is used, replace $\mathbf{x}_i$ with $\mathbf{r}_i$).
\State Train another GNN ($\texttt{GNN}_{\mathbf{\hat{\Phi}}}$) on $\mathbf{E}=[\mathbf{e}_1, \ldots, \mathbf{e}_{m+n}]$, $\mathbf{A}$, and $\mathbf{Y}_L$ with loss function in Eq. \ref{gnn_train}.
\State Classify unlabeled nodes using $\texttt{GNN}_{\mathbf{\hat{\Phi}}}$: $\overline{\mathbf{y}}_j = \operatorname{argmax}(\texttt{GNN}_{\mathbf{\hat{\Phi}}}(\mathbf{e}_j, \mathbf{A}))$, $j \in \{1, \cdots, n\}$.
\State \Return $\mathbf{Y}_U=[\overline{\mathbf{y}}_1, \cdots, \overline{\mathbf{y}}_n]$.
\end{algorithmic}
\label{algorithm1}
\end{algorithm}

\begin{table*}[!tb]
    \footnotesize
    \caption{Performance comparison of different models using complete labels.}
    \centering
    \setlength{\tabcolsep}{1.6mm}{ 
    \begin{tabular}{l|cccc}
        \toprule
        Dataset  
        & \texttt{Cora(140)} & \texttt{Cora(1208)} & \texttt{PubMed(60)} & \texttt{PubMed(18217)}
        \\ 
        \midrule
        \texttt{GCN} 
        & \textcolor{black}{${0.6038}_{\pm0.0116}$} 
        & \textcolor{black}{${0.7182}_{\pm0.0055}$} 
        & \textcolor{black}{${0.6770}_{\pm0.0217}$} 
        & \textcolor{black}{${0.8092}_{\pm0.0057}$} 
        \\ 
        {\texttt{DeBERTa}} 
        & \textcolor{black}{${0.3048}_{\pm0.0198}$} 
        & \textcolor{black}{${0.7366}_{\pm0.0137}$} 
        & \textcolor{black}{${0.3616}_{\pm0.0976}$} 
        & \textcolor{black}{${0.9510}_{\pm0.0056}$} 
        \\ 
        {\texttt{DeBERTa+GPT}} 
       & \textcolor{black}{${0.5450}_{\pm0.1121}$} 
        & \textcolor{black}{${0.7430}_{\pm0.0067}$} 
        & \textcolor{black}{${0.4514}_{\pm0.0700}$} 
        & \textcolor{black}{${0.9388}_{\pm0.0020}$} 
        \\ 
        {\texttt{SFGL}}
        & \textcolor{black}{${0.6314}_{\pm0.0191}$} 
        & \textcolor{black}{${0.7278}_{\pm0.0151}$} 
        & \textcolor{black}{${0.7238}_{\pm0.0124}$} 
        & \textcolor{black}{${0.9382}_{\pm0.0078}$} 
        \\ 
        {\texttt{SFGL+GPT}}
        & \textcolor{black}{${0.6856}_{\pm0.0131}$} 
        & \textcolor{black}{${0.7396}_{\pm0.0105}$} 
        & \textcolor{black}{${0.8640}_{\pm0.0381}$} 
        & \textcolor{black}{${0.9400}_{\pm0.0019}$} 
        \\ 
        \bottomrule
    \end{tabular}}
    \label{complete_lables}
\end{table*}

\subsection{Training Strategies}
\label{trainingstrategies}
The training algorithms of \texttt{SFGL}, \texttt{SFGL+GPT}, \texttt{SFGL+\underline{GCN}} and \texttt{SFGL+\uwave{GCN}} are outlined in Algorithms \ref{algorithm1} and \ref{algorithm2}. The primary difference between \texttt{SFGL} and \texttt{SFGL+GPT} lies in steps 5 and 7, where \texttt{SFGL+GPT} incorporates GPT responses to enhance text embeddings. Similarly, the distinction between \texttt{SFGL+\underline{GCN}} and \texttt{SFGL+\uwave{GCN}} is that the former employs a LM to generate text embeddings, while the latter further refines them using an LLM. Specifically, \texttt{SFGL+\uwave{GCN}} replaces the original node textual sequences with GPT-generated responses for finetuning.
All these models use a GCN as the final classifier.

\subsection{Limitation}
Due to the inherent differences in the formation mechanisms of KNN graphs and real-world networks, constructing a fully exact scale-free KNN graph remains challenging. A potential solution is to incorporate the scale-free property as a constraint within an optimization objective and then optimize it to learn an exact latent graph, which presents an intriguing direction for future research.
Furthermore, this work primarily focuses on citation networks. Extending our findings and methodologies to other types of networks, such as social and biological networks, represents another promising avenue for future exploration.

\begin{algorithm}[!tb]
\caption{Scale-Free Graph-Language Model (\texttt{SFGL+\underline{GCN}} / \texttt{SFGL+\uwave{GCN}})}
\begin{algorithmic}[1]
\State Return to step $1$ of Algorithm \ref{algorithm1}, replace $\mathbf{F}$ with $\mathbf{E}$ generated in step $7$.
\State Repeat steps $2, 3$, and $4$ of Algorithm \ref{algorithm1} once to generate better pseudo-labels.
\State Repeat steps $6, 7, 8$, and $9$ of Algorithm \ref{algorithm1} once to re-classify the unlabeled nodes.
\end{algorithmic}
\label{algorithm2}
\end{algorithm}

\begin{figure*}[!tb]
    \subfigure[\texttt{Cora}]{\includegraphics[scale=0.245,trim=0 0  50 45,clip]{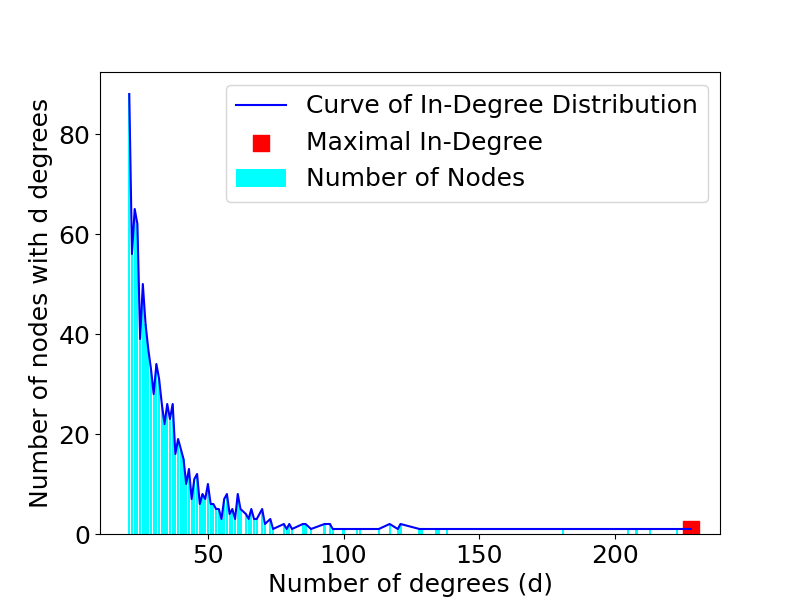}\label{cora20}} 
    \subfigure[\texttt{Pubmed}]{\includegraphics[scale=0.245,trim=0 0  50 45,clip]{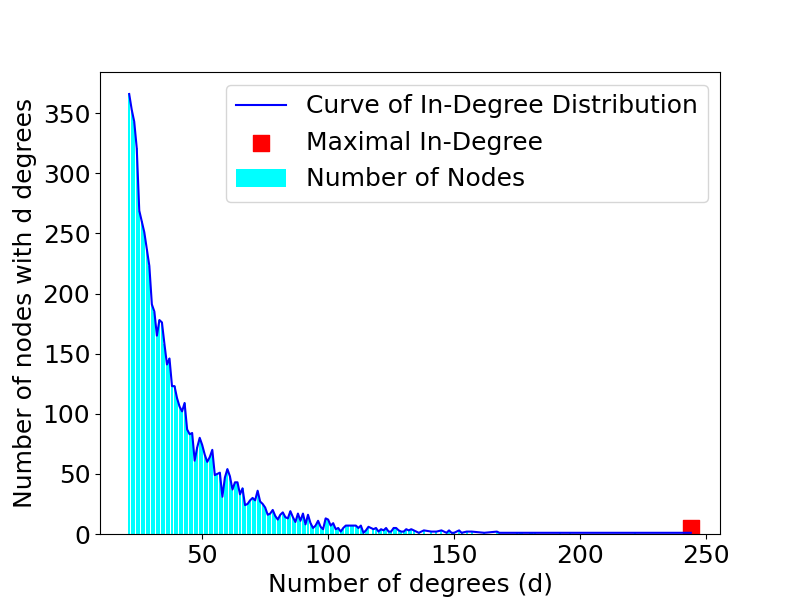}\label{pubmed20}} 
    \subfigure[\texttt{ogbn-arxiv}]{\includegraphics[scale=0.245,trim=0 0  50 45,clip]{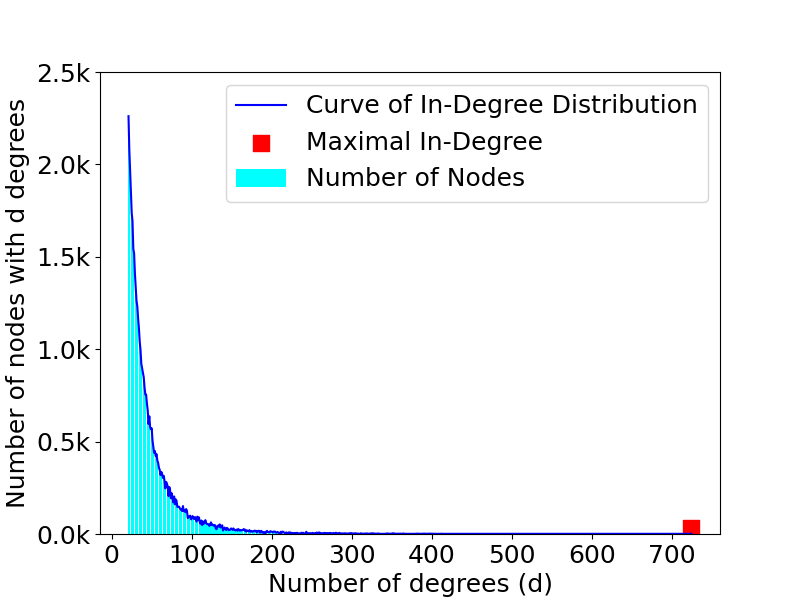}\label{ogbn-arxiv20}}
    \vspace{-2mm}
    \caption{In-degree distribution of KNN graphs on different datasets.}
    \label{in_degree_distribution_a1}
\end{figure*}

\begin{figure*}[!tb]
     \subfigure[\texttt{Cora}] {\includegraphics[scale=0.245,trim=0 0  50 50,clip]{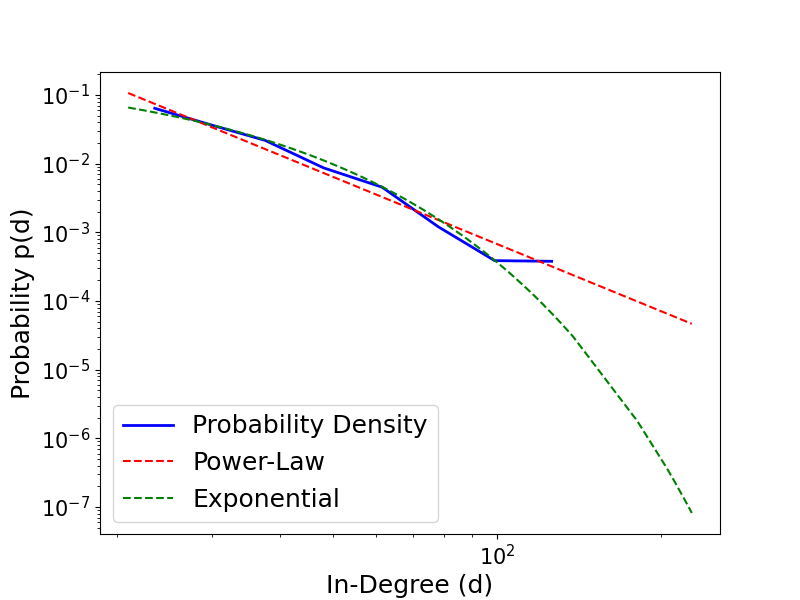}\label{fig:cora20}}
    \subfigure[\texttt{Pubmed}]{\includegraphics[scale=0.245,trim=0 0  50 50,clip]{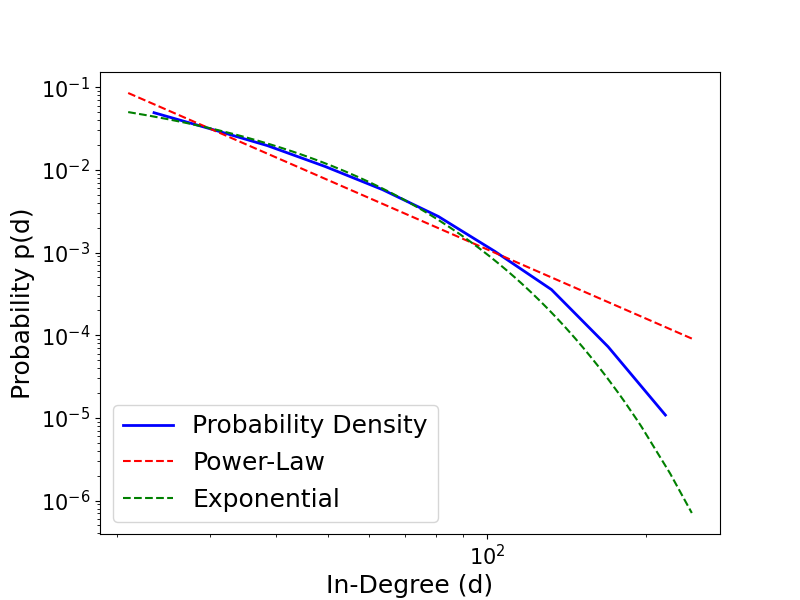}\label{fig:pubmed20}}
     \subfigure[\texttt{ogbn-arxiv}]{\includegraphics[scale=0.245,trim=0 0  50 50,clip]{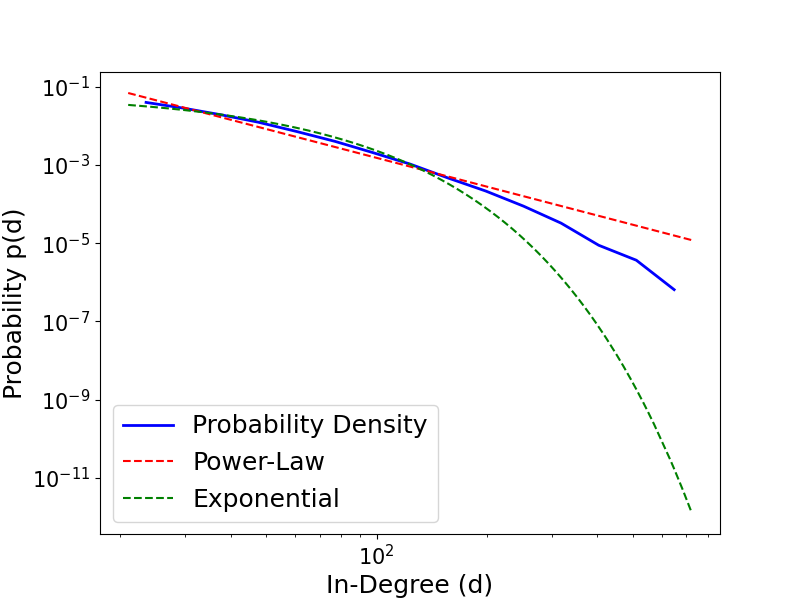}\label{fig:ogbn-arxiv20}}
      \vspace{-2mm}
     \caption{Curve fitting to the in-degree distribution of KNN graphs on different datasets. }
     \label{fitting_curves_a1}
 \end{figure*}

\begin{figure*}[!tb]
    \centering
    \subfigure[\texttt{SLAPS}]{\includegraphics[scale=0.35,trim=0 0  50 45,clip]{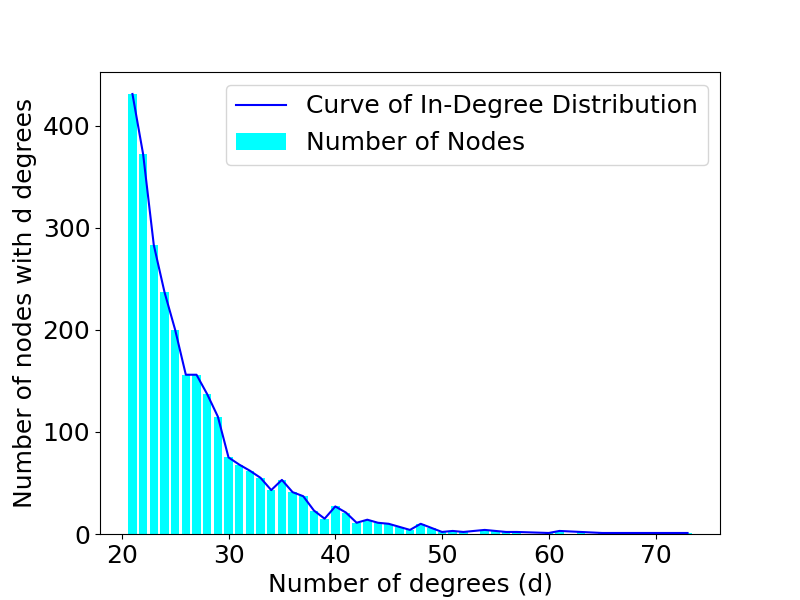}} 
    \subfigure[\texttt{Euclidean}]{\includegraphics[scale=0.35,trim=0 0  50 45,clip]{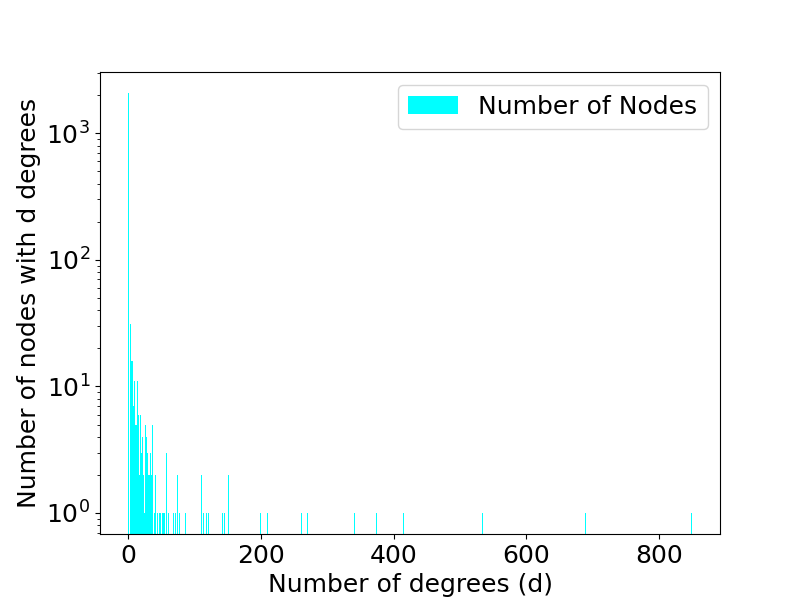}}
    \caption{In-degree distribution of graph generated by (a) SLAPS and (b) KNN (Euclidean).}
    \label{other_distribution}
\end{figure*}

\end{document}